\documentclass{article}

\usepackage{arxiv}

\usepackage{amsmath,amsfonts,bm}

\def\eqref#1{equation~\ref{#1}}

\def\1{\bm{1}}

\DeclareMathAlphabet{\mathsfit}{\encodingdefault}{\sfdefault}{m}{sl}
\SetMathAlphabet{\mathsfit}{bold}{\encodingdefault}{\sfdefault}{bx}{n}

\usepackage{natbib}
\usepackage{hyperref}
\usepackage{url}

\usepackage{amsfonts}       
\usepackage{amsmath}
\usepackage{amssymb}
\usepackage{amsthm}
\usepackage{bm}
\usepackage{bbm}
\usepackage{mathtools}
\usepackage{mathrsfs}
\usepackage{nicefrac}       
\usepackage[thinc]{esdiff}
\usepackage{hyperref}

\usepackage{graphicx}
\usepackage{subfigure}
\usepackage{float}
\usepackage{tikz}
\usetikzlibrary{bayesnet}

\usepackage{array}
\usepackage{longtable}
\usepackage{multirow}
\usepackage{multicol}
\usepackage{adjustbox}
\usepackage{makecell}
\usepackage{arydshln}
\usepackage{booktabs}       

\usepackage{xcolor}

\usepackage{url}            
\usepackage{algorithm} 
\usepackage{algorithmic}
\usepackage{verbatim}

\usepackage{blindtext}
\usepackage{microtype}
\usepackage{titlesec}
\usepackage[utf8]{inputenc}
\usepackage{framed}
\usepackage[english]{babel}
\usepackage{comment}
\def\CL{CL}
\newcommand{\SCL}{SCL}
\newcommand{\UCL}{UCL}

\newcommand{\BlackBox}{\rule{1.5ex}{1.5ex}}  
\ifdefined\proof
    \renewenvironment{proof}{\par\noindent{\bf Proof\ }}{\hfill\BlackBox\\[2mm]}
\else
    \newenvironment{proof}{\par\noindent{\bf Proof\ }}{\hfill\BlackBox\\[2mm]}
\fi
 
\newtheorem{theorem}{Theorem}
\newtheorem{lemma}{Lemma} 
 
\newtheorem{remark}{Remark}
\newtheorem{corollary}{Corollary}
\newtheorem{definition}{Definition}

\usepackage{bbm}
\def\Pr{\ensuremath{\mbox{Pr}}}

\def\EE{\ensuremath{{\mathbb E}}}

\def\RR{\ensuremath{{\mathbb R}}}
\def\ZZ{\ensuremath{{\mathbb Z}}}
\def\Acal{\ensuremath{{\mathcal{A}}}}
\def\Ccal{\ensuremath{{\mathcal{C}}}}
\def\Dcal{\ensuremath{{\mathcal{D}}}}
\def\Ecal{\ensuremath{{\mathcal{E}}}}
\def\Fcal{\ensuremath{{\mathcal{F}}}}
\def\Mcal{\ensuremath{{\mathcal{M}}}}

\def\Scal{\ensuremath{{\mathcal{S}}}}
\def\Vcal{\ensuremath{{\mathcal{V}}}}
\def\Xcal{\ensuremath{{\mathcal{X}}}}

\def\Zcal{\ensuremath{{\mathcal{Z}}}}
\def\bfzero{\ensuremath{{\mathbf{0}}}}
\def\bfones{\ensuremath{{\mathbf{1}}}}

\newcolumntype{C}[1]{>{\centering\arraybackslash}p{#1}}

\newcommand{\cupdot}{\mathbin{\mathaccent\cdot\cup}}

\allowdisplaybreaks

\usepackage{xcolor}
\usepackage{bbm}
\newif\ifcomments
\commentstrue

\ifcomments
    \usepackage{ulem}
    \def\picomment#1{{$\!$\color{magenta} [PI: #1]}}
    
\else 
    \def\picomment#1{}
    
\fi

\ifcomments
    \usepackage{ulem}
    \def\tncomment#1{{$\!$\color{orange} [TN: #1]}}
    
\else 
    \def\tncomment#1{}
    
\fi

\ifcomments
    \usepackage{ulem}
    \def\sacomment#1{{$\!$\color[rgb]{0.0,0.5,0.0} [SA: #1]}}
    
\else 
    \def\sacomment#1{}
    
\fi
\ifcomments
    \usepackage{ulem}
    \def\rbcomment#1{{$\!$\color[rgb]{0.55,0.71,0.0} [RB: #1]}}
    
\else 
    \def\rbcomment#1{}
    
\fi

\title{Optimal Representations for Generalized Contrastive Learning with Imbalanced Datasets}

\author{
Thuan Nguyen \\
  Department of Engineering, Engineering Technology\\
  East Tennessee State University\\
  \texttt{nguyent11@etsu.edu} \\
   \And
 Shuchin Aeron \\
  Department of Electrical and Computer Engineering\\ 
      Tufts University\\
  \texttt{shuchin@ece.tufts.edu} \\
  \And
D. Richard Brown III \\
  Department of Electrical and Computer Engineering\\ 
      Worcester Polytechnic Institute\\
  \texttt{drb@wpi.edu} \\
  \AND
  Prakash Ishwar \\
  Department of Electrical and Computer Engineering \\ Boston University\\
  \texttt{pi@bu.edu} \\
}

\begin{document}

\maketitle

\begin{abstract} 
In this paper, we provide a computable characterization of the geometry of optimal representations in Contrastive Learning (CL) when the classes are imbalanced. When classes are balanced and the representation dimension is greater than the number of classes, it is well-known that the optimal representations exhibit Neural Collapse (NC), i.e., representations from the same class collapse to their class means and the class means form an Equiangular Tight Frame (ETF). For imbalanced classes and a large, generalized family of \CL{} losses, we prove that the optimal representations of all samples from the same class collapse to their class means and their geometry exhibits an angular symmetry structure that is determined by the relative class proportions. In general, we show that the geometry can be determined by solving a convex optimization problem. Exploiting this symmetry structure, we analytically investigate a special case where class imbalance is extreme and prove that \CL{} exhibits a phenomenon called Minority Collapse (MC) where all samples from the minority classes (classes with small probabilities) collapse into a single vector, whenever the class imbalance exceeds a threshold, which in turn depends on the regularity properties of the \CL{} loss used and on the number of negative samples. Numerical results are provided to illustrate these phenomena and corroborate the theoretical results. We conclude by identifying a number of open problems.  
\end{abstract}


\section{Introduction}
\CL{} is a machine learning technique that aims to learn a representation map by pulling ``similar'' samples closer together while simultaneously pushing apart ``different'' samples in the representation space. These representations can then be directly utilized or fine-tuned for downstream tasks. Over the past decade, \CL{} has received significant attention due to its applications ranging from computer vision, time series analysis, and natural language processing (see \cite{jaiswal2020survey} for a comprehensive survey).

In \CL{} terminology, a reference sample is called the ``anchor'' sample, a sample similar to it is called the ``positive'' sample, and a sample different from it is called the ``negative'' sample. If label information is not available (unsupervised setting), positive samples are usually constructed via data augmentations of the anchor, and negative samples are randomly selected from the dataset \cite{chen2020simple}. When label information is available (supervised setting), positive samples can be selected from the same class as the anchor while negative samples can be picked from either (a) classes other than the anchor's class \cite{jiang2024supervised,jiang2024hard}, or (b) any class (including the anchor's class) \cite{khosla2020supervised}. Under a suitable model of the data generating the positive and negative samples in the unsupervised as well as supervised settings, the aim of this paper is to characterize the optimal representations learned via \CL{} under an \textit{unconstrained features model} wherein the \CL{} map is assumed to have adequate capacity to realize any mapping. This is an important problem that sheds light on the effect of positive and negative sampling mechanisms in \CL. In the next section, we will begin by reviewing related work and outline our main contributions in that context.

\subsection{Limitations of related work and contributions}

\textbf{Loss function, sampling distribution, and number of negative samples per positive-pair $k$:} To the best of our knowledge, most theoretical studies of \CL{} that have aimed to understand the structure of optimum  representations \cite{fang2021exploring,graf2021dissecting,kothapalli2023neural,kini2024symmetric,behnia2024supervised}  have done so \textbf{only for \textit{empirical} versions of the 
InfoNCE \CL{} loss (or its variants) with norm-bounded representation constraints}
where within each mini-batch $b$, consisting of $n_b$ of samples, the anchor is \textbf{\textit{uniformly}} distributed over \textbf{\textit{all $n_b$ samples}}, the positive sample is \textbf{\textit{uniformly}} distributed over \textbf{\textit{all $n_b$ samples}} (some works exclude the anchor), and for each anchor-positive pair, \textbf{\textit{all $n_b$ samples}} (some works exclude the anchor or/and the positive sample) are negative samples (i.e., $k=n_b$ or $n_b-1$ or $n_b-2$). 
Unraveling the impact of $k$ is not possible with the approaches taken in extant works since they only consider empirical \CL{} losses where $k$ is nearly equal to the batch size.

\textbf{Class proportions:} In addition to heavily focusing on the empirical InfoNCE  loss together with the (nearly) maximum possible range of $k$, almost all prior theoretical works in \CL{}  \cite{fang2021exploring,graf2021dissecting,kothapalli2023neural,behnia2024supervised} have focused on the  \textbf{\textit{idealized balanced setting}}  in which each sample belongs to one of $C>1$ classes (or latent classes) and \textbf{\textit{all classes are equally likely,}} i.e., have the same sample size in the training set. The more realistic and practically useful \textbf{\textit{unbalanced setting}} has been analyzed primarily for \textbf{\textit{classifier networks}} with the empirical Mean Squared Error (MSE) loss \cite{pmlr-v202-dang23b} and empirical cross-entropy loss \cite{hong2024neural,pmlr-v235-dang24a} where there is an additional linear classifier layer following the representation mapping and the loss function explicitly depends on the labels of the samples.  Analysis of the unbalanced case for \CL{} is very limited and confined to the empirical InfoNCE loss \cite{fang2021exploring,kini2024symmetric,behnia2024supervised}. 

\textbf{Minority-Collapse (MC) phenomenon:} When classes are not balanced, the representations of all the samples in several distinct \textit{minority} classes (classes with small probabilities) may collapse into a single vector. This phenomenon has been studied  only fairly recently, primarily within the context of \textit{classifier} networks with either empirical MSE loss \cite{pmlr-v202-dang23b} or empirical cross-entropy loss \cite{hong2024neural,pmlr-v235-dang24a}. Within the \CL{} context, the \textit{existence} of minority-collapse was proved in \cite{fang2021exploring} \textit{only in the asymptotic limit where the minority class probabilities vanish}. 

This paper makes the following contributions:
\begin{enumerate}
\setlength \itemsep{-3pt}
    \item We construct a novel lower bound (Lemma~\ref{lemma:CLriskLB}) that holds for the general family of \CL{} losses that are based on functions that are strictly convex and argument-wise strictly increasing and allow any value of $k$ (the number of negative samples per positive-pair).  
    This subsumes and generalizes popular loss functions such as the InfoNCE loss function. 
    The bound is a convex function of the Gram matrix whose entries are the pairwise inner products of the class mean feature vectors. 
    We also derive the asymptotic limit of the lower bound for the InfoNCE loss function when $k\uparrow \infty$ (Corollaries~\ref{cor:InfoNCElb} and  \ref{cor:InfoNCElb_infinite_k}).
    
    \item When the representation dimension $d \geq C-1$, we prove that \textit{\textbf{the lower bound has a unique minimizer which is rank-deficient with a unit-constant principal diagonal}} (Lemmas~\ref{lemma:CLriskLB}~--~\ref{lemma:rank_of_opt_Gram_mtx} and Theorem~\ref{theorem: 1}).
    We also show that \textit{\textbf{the generalized \CL{} loss is minimized when there is intra-class variance-collapse,}} \textit{i.e.}, when the feature vectors of all the samples from the same class are identical (Corollary~\ref{cor:varcollapse}). However, the geometry of the optimal class feature vectors need not form an Equiangular Tight Frame (ETF) as in the balanced classes scenario. We show that the optimal geometry can be numerically computed as the solution to a convex program (Remark~\ref{remark: 2}). 

    \item We prove that the geometric structure of the optimal class means exhibits a key equiangular symmetry structure that is determined by the relative class proportions (Theorem~\ref{theorem:equiangular-properties} and Corollary~\ref{cor: 1}). We further show that these properties are consistent with corresponding results for balanced classes and \textit{\textbf{resolve a question that was left open in \cite{jiang2024hard}}, namely whether the ETF geometry is optimal when the positive pairs are not conditionally independent given their class label and the classes of the positive and negative samples can collide (Remark~\ref{remark: 5}).} 

    \item We further investigate the case when the class imbalance is extreme and prove that \CL{} exhibits the MC phenomenon in the scenario where there is one majority class and equiprobable minority classes with 
    the minor class probability less than \textbf{\textit{a non-asymptotic threshold}} $\tau$ that depends on the number of classes, the number of negative samples per anchor, and bounds on the norms of the subgradients of the \CL{} loss function (Lemmas~\ref{lemma:Astarform}~--~\ref{lemma:Astarbound} and Theorem~\ref{theorem:minoritycollapse}). Specializing to the InfoNCE loss function yields conservative \textit{\textbf{parameter-free}} thresholds 
    $\tau = 0.9292$ (Corollary~\ref{cor:InfoNCElambda1threshold}) and $\tau = 0.9438$ (Corollary~\ref{cor:SCLInfoNCElambda1threshold} in Appendix~\ref{apd:appendix_proofs}) in different negative sampling settings. 
    \item Finally, we prove that all the above results hold under two different negative sampling settings: (1) Unsupervised \CL{} (\UCL), where the negative samples are selected from the whole dataset including samples from the same class as that of the anchor and (2) Supervised \CL{} (\SCL), where the negative samples are selected from classes that are different from that of the anchor. 
\end{enumerate}

The remainder of this paper is structured as follows. Section~\ref{sec: problem setup} formally introduces the \CL{} framework and formulates the core optimization problem of interest. A tight lower bound for the generalized contrastive loss (and the $k\uparrow \infty$ asymptotic limit for the InfoNCE loss) that is a function of the mean feature vectors of the classes, together with necessary and sufficient conditions for equality, is established in  Section~\ref{sec: lower bound}. 
That the lower bound is a strictly convex function of the Gram matrix whose entries are the pairwise inner products of unit-norm class mean feature vectors, the necessity and sufficiency of intra-class variance-collapse for optimality, and the complete characterization of the optimal rank-deficient class means when $d \geq C-1$ are all established in Section~\ref{sec: main result}.  
Equiangular symmetry properties of the optimal class means and their implications are  established in Section~\ref{sec: equiangular}.
The MC phenomenon is investigated in Section~\ref{sec: minority collapse} where a non-asymptotic threshold for MC is derived. 
Numerical experiments that corroborate and illustrate our theoretical results appear in Section~\ref{sec: numerical results}. 
We end with a discussion of open questions in Section~\ref{sec: conclusion}.
Proofs of theoretical results are presented in Appendix~\ref{apd:appendix_proofs}.

\noindent\textbf{Notation:} For $i, j \in \ZZ$, $i < j$, we define $i:j := i, i+1, \ldots, j$ and  $a_{i:j} := a_i, a_{i+1}, \ldots, a_j$. If $i > j$, $i:j$ and $a_{i:j}$ are void expressions. We will denote the ``all zeros'' and ``all ones'' column vectors by $\bfzero$ and $\bfones$, respectively. The dimensions of $\bfzero$ and $\bfones$ will be clarified within each context they are used. 

\section{Contrastive learning problem setup and notation}
\label{sec: problem setup}

Let $\Xcal \subseteq \mathbb{R}^{d'}$ denote the data space, 
$f: \Xcal \rightarrow \Zcal$ a representation 
function from data space to representation space (or feature space) $\Zcal \subseteq \mathbb{R}^d$, and $\Fcal$ a (parameterized) family of such representation functions such as a those specified by a deep neural network with a specified architecture.
Contrastive Learning (\CL) is based on tuples $(x, x^+, x^{-}_{1:k}) \sim p(x,x^{+},x^{-}_{1:k})$,  where 
\begin{enumerate}
\setlength \itemsep{-3pt}
    \item 
    $x$ is called the anchor (or context),
   \item  
   $x^+$ the positive sample (relative to the given anchor $x$), and
   \item 
   $x^{-}_{1:k}, \,k \geq 1$, the negative examples (relative to the given anchor $x$).
\end{enumerate}
The anchor $x$ is also regarded as a positive sample and $(x,x^+)$ is called a positive pair. The objective of \CL{} is to learn a mapping $f \in \Fcal$ via solving the following optimization problem,
\setlength{\intextsep}{0pt}
\begin{align}
    \arg\min_{f \in \Fcal} L(f), \quad 
    L(f) := \EE\left[ \ell_k\left(x, x^{+}, x^{-}_{1:k},f\right)\right], 
    \label{eq: objective function}
%
\end{align} 
where $L(f)$ is the \CL{} risk of a representation function $f$ with the expectation $\EE[\cdot]$ (or empirical average) taken with respect to the joint distribution (or empirical distribution) $p(x,x^+,x^{-}_{1:k})$ and 
$\ell_k(\cdot)$ is a \CL{} loss function that encourages \textit{alignment} between the positive pairs $(x,x^+)$ in representation space, as measured by the inner product $f(x)^\top f(x^+)$, and discourages the alignment between the $k$ negative pairs $(x,x^{-}_i), i = 1:k$, in representation space, as measured by the inner products $f(x)^\top f(x^{-}_i), i = 1:k$.\footnote{\label{foot:dotproductsimilarity} In Contrastive Learning, the feature vectors are typically normalized to have unit Euclidean length. Then, the inner product of two feature vectors is larger if, and only if, they are closer to each other in Euclidean distance. Therefore, the inner product of two feature vectors acts as an ``inverse distance'' or similarity measure between them.}
The representation map learned via \CL{} is treated as a pre-trained feature extractor and is used either directly or with fine-tuning in various downstream supervised tasks, predominantly classification.  

In this work, we establish results that hold in great generality for the entire family of \CL{} loss functions proposed in \citep{jiang2024hard} as defined below.
\begin{definition}[Generalized \CL{} Loss Function] \label{def: 1}
A Generalized \CL{} loss function is of the form
\begin{align}
\ell_k(x, x^+ , x^{-}_{1:k}, f) := \psi\big(f(x)^\top (f(x^-_1) - f(x^+)), \ldots, f(x)^\top(f(x^-_k) - f(x^+)\big)
\label{eq:genCLloss}
\end{align}where $\psi: \RR^k \rightarrow \RR$ is a function which is strictly convex and argument-wise strictly increasing (\textit{i.e.}, strictly increasing with respect to each argument when the other $k-1$ arguments are held fixed).\footnote{As a technical aside, the function $\psi$ is a so-called \textit{proper} convex function because its range is $\RR$ which excludes $-\infty$.} The value of $k$ in not restricted. 
\end{definition}
We note that this subsumes and generalizes popular loss functions with spherical-ball normalized 
representations including the popular InfoNCE loss function defined in Appendix~\ref{apd: 1} and its variants
(InfoLOOB, N-pair, Decoupled Contrastive Loss, etc.) which have been widely used.\footnote{The sigmoid loss does not satisfy Definition~\ref{def: 1}. Triplet
loss corresponds to choosing 
$\psi(t_1,\ldots,t_k) = \sum_{i=1}^k
\max\{t_i + \alpha, 0\}, \alpha > 0$. The $\psi(\cdot)$ function
here is convex, but not strictly convex. All results in this paper, except those related to the uniqueness of the
minimizer, also hold for the triplet loss.} 
We focus on the general family in Definition~\ref{def: 1}  to highlight that all results presented in this paper only rely on two key properties of the \CL{} loss function, namely convexity and monotonicity, and nothing else specific to a particular loss function like InfoNCE.


Unlike prior works which are restricted to the empirical \CL{} risk where the joint distribution of the anchor, positive and negative samples (and also their latent labels in many works) are uniform over suitable discrete subsets, we adopt a general distributional perspective throughout and work with the population risk (which subsumes the empirical risk as a special case when the distribution is empirical)  with the following  key modeling assumptions that are consistent with the specialized assumptions on the (empirical)  distribution of samples in prior works: \\
\noindent\textbf{A1: Class labels.} The samples have associated labels given by a deterministic labeling function $y(\cdot):\Xcal \rightarrow \Ccal := \{1,\ldots, C\}, C > 1$. These labels represent classes in the supervised setting and latent, i.e., hidden, classes or clusters in the unsupervised setting. \\
\noindent\textbf{A2: Positive samples.} The joint distribution of positive samples is such that they have the same label. This can be ensured by design in the supervised setting, but in the unsupervised setting this is an assumption on the method used to sample a positive pair, e.g., an augmentation mechanism. \\
\noindent\textbf{A3: Joint distribution.}
Let $x,x^+ \in \Xcal$ be a pair of positive samples and $y \in \Ccal$ their common class label. Let  $x^-_{1:k} \in \Xcal$ be a set of $k$ negative samples associated with the positive pair and $y^-_{1:k} \in \Ccal$ their respective class labels. In the \UCL{} setting where the negative samples are chosen from the entire dataset, including possibly from the class of the positive pair, the joint distribution of all $(k+2)$ samples $x,x^+,x^-_{1:k}$ and their $(k+1)$ labels $y,y^-_{1:k}$ has the following form 
\begin{align}
p(x,x^+,x^-_{1:k},y,y^-_{1:k}) &=  p(y,y^-_{1:k}) \, p(x,x^+,x^-_{1:k}|y,y^-_{1:k})\ , \nonumber \\
p(y,y^-_{1:k}) &= \lambda_y\prod_{t=1}^k \lambda_{y^-_t}\ , \label{eq:joint1of2} \\
p(x,x^+,x^-_{1:k}|y,y^-_{1:k}) &= q(x,x^+|y) \prod_{t=1}^k s(x^-_t|y^-_t), \label{eq:joint2of2}
\end{align}
where $\lambda_{1:C} \in (0,1), \sum_{i \in \Ccal} \lambda_i = 1$, denote the probabilities (or relative sample proportions) of the $C$ possible classes \textbf{\textit{and they need not be balanced}}, $q(x,x^+|y)$ is the conditional distribution of a positive pair given their label, and $s(x^-|y^-)$ is the conditional distribution of a negative sample given that it is from class $y^-$.  

We note that 
$(x^-_1,y^-_1),\ldots,(x^-_k,y^-_k)$ are independent and identically distributed (iid) and also independent of $(x,x^+,y)$. The $(k+1)$ labels $y,y^-_{1:k}$ are iid which implies that, with non-zero probability, negative samples could have the same label as that of the positive pair, an event referred to as ``class collision''. Moreover, 
$x^-_{1:k}$ are conditionally iid given $y^-_{1:k}$, but unlike in \citep{jiang2024hard}, we do not assume that $(x,x^+)$ are conditionally independent given their label $y$. \\
We focus on the \UCL{} setting to establish all results. In Appendix~\ref{apd:extensions_to_SCL} we discuss how \textit{all} our theoretical results continue to hold, with minor adjustments to some expressions, in the \SCL{} setting where 
the negative samples are chosen from classes other than that of the positive pair, i.e., $y^-_{1:k} \in \Ccal \setminus \{y\} \text{ w.p.1}$ and the anchor and positive sample are conditionally iid given their class. Then,
$p(y,y^-_{1:k})$ in (\ref{eq:joint1of2}) is changed to
\begin{equation}
p_{SCL}(y,y^-_{1:k}) := \lambda_y\prod_{t=1}^k \left(\frac{\lambda_{y^-_t}}{1-\lambda_y}\right). \label{eq:SCL_joint_pmf_labels} 
\end{equation}
and $q(x,x^+|y) = s(x|y)\,s(x^+|y)$.\\
\noindent\textbf{A4: Marginal conditional distributions.} 
As in \citep{jiang2024hard} and for analytical simplicity we also assume that 
\[
\forall i \in \Ccal, \forall x, x^+ \in \Xcal, \quad p(x|y=i) = s(x|i),\quad p(x^+|y=i) = s(x^+|i),
\]
i.e., the \textit{marginal} conditional distributions of $x^+$ given $y=i$ and $x$ given $y=i$ are both $s(\cdot|i)$ which is the marginal conditional distribution of a negative sample $x^-$ given $y^- = i$. 
This assumption can be ensured in the supervised setting, since labels are available. This also holds in the unsupervised setting, if a negative sample is generated using the same sampling mechanism that was used to generate a positive sample, e.g., via an augmentation of a reference sample.
Indeed, in practical implementations \cite{chen2020simple,khosla2020supervised, jiang2024hard}, all samples in a mini-batch are first augmented using the same family of random augmentations and then the anchors, positives, and negatives are selected from these. Thus, negative samples are generated using the same augmentation-based sampling mechanism used to generate the positive pair. Consequently, the \textbf{\textit{marginal}} conditional distributions of the positives and negatives are the same. We refer the readers to \cite{jiang2024hard} for a more detailed analysis.
Under this assumption, for a representation function $f$ and all $j \in \Ccal$, if we let $\mu_j$ denote the mean of class $j$ samples in the representation space, then we have
\begin{align}
\forall j\in \Ccal, \forall i \in \{1:k\},\quad \mu_j &= \EE[f(x)|y = j] = \EE[f(x^+)|y = j] = \EE[f(x^-_i)|y^-_i=j]. \label{eq:classmeans1of3} 
\end{align}
We define $M$ as the $d \times C$ matrix of class means in representation space, specifically,
\begin{align*}
M &:= \left[\mu_1\ \mu_2 \cdots  \mu_C \right].  
\end{align*}
\noindent\textbf{A5: Spherical-ball normalized representations.}  All prior theoretical studies of \CL{} constrain the norms of the representations. This is a type of feature-normalization  which typically improves 
the performance of \CL{} in practice \cite{wang2020understanding} 
and also makes the inner product a truer measure of ``inverse distance'' (see footnote \ref{foot:dotproductsimilarity}). 
This can be done by explicitly requiring all representation maps in $\Fcal$ to be norm-bounded for all samples, or implicitly by adding a quadratic penalty on the representation norms of the anchor, positive, and negative samples to the loss function. In our work, we will adopt the direct approach by requiring all representation functions to have a $2$-norm less than or equal to one for all samples, i.e.,
\[
\Fcal = \{f: \forall x \in \Xcal,\,  \|f(x)\|^2 := f^\top(x) f(x) \leq 1\}.
\]
Thus, $\Fcal$ is the family of all representation functions that are norm-bounded, but otherwise unconstrained. 
Note that since the representation vectors are confined to the unit ball, i.e., $\forall x \in \Xcal, \|f(x)\|^2 \leq 1$, from the Cauchy-Schwarz inequality (or alternatively by the convexity of the squared norm function $\|\cdot\|^2$), we must have 
{\small
\begin{equation}
\forall j \in \Ccal,\ 
\|\mu_j\|^2 = \|\EE[f(x^-)|y^- = j ]\|^2 =  \|\EE[f(x^+)|y = j ]\|^2 = \|\EE[f(x)|y = j ]\|^2 \leq \EE [\|f(x)\|^2 | y = j] \leq 1. 
\nonumber
\end{equation}
}
\normalsize
\noindent\textbf{A6: Unconstrained Features Model (UFM).} In practice, the family of representation functions $\Fcal$ is further constrained to be representable by a neural network having a specific architecture. The optimal solutions of the optimization problem in (\ref{eq: objective function}) will be included in such a family if the representation capacity of the neural network is sufficiently large, i.e., the neural network can approximate an arbitrary mapping $f: \Xcal \rightarrow \RR^d$ to any desired accuracy.
Almost all prior theoretical studies of \CL{} use UFM  \cite{fang2021exploring,graf2021dissecting} which treats a neural network’s final-layer feature vectors, denoted by $z = f(x)$, as the free optimization variables instead of the network weights. This decouples feature geometry from the complex nonlinear encoder weight parameterization. UFM is used as an analytically tractable proxy for deep neural networks with a sufficiently high representation capacity. In this work we will also use UFM with the class of generalized \CL{} loss functions
\[
\ell(z, z^+ , z^{-}_{1:k}) = \psi(z^\top (z_1^- - z^+),  \cdots, z^\top(z_k^- - z^+)) 
\]
where $z = f(x), z^+ = f(x^+), z_1^- = f(x_1^-), \ldots, z_k^- = f(x_k^-)$.

The optimization problem in (\ref{eq: objective function}) was solved for special loss functions in the balanced dataset setting, i.e., $\lambda_1 = \lambda_2 = \ldots = \lambda_C = 1/C$, in \cite{jiang2024hard, wang2023towards}, where the optimal solution was shown to exhibit NC. Characterizing and computing the optimal solutions for imbalanced datasets was left open and is the primary focus of this work.

In Section~\ref{sec: lower bound}, we will construct a tight lower bound for the generalized contrastive risk as a function of the class means, and then optimize this lower bound to find the optimal class means in Section~\ref{sec: main result}. 

\section{Tight lower bound for \CL{} risk in terms of class means}\label{sec: lower bound}
Our first key result is the following lemma which shows that it is possible to lower bound the contrastive risk by a function of the class means in representation space. Furthermore, this bound can be attained by any representation function $f$ which collapses the representations of all samples within a class to the class mean and if all class means have unit norm.  The lemma also shows that in order to achieve the lower bound, ``intra-class variance-collapse'', i.e., the collapse of the representations of all samples from the same class to their class mean, and unit norm class means are also necessary to attain the lower bound. In the next section, we will characterize the optimal class means that minimize the lower bound.
\begin{lemma}\label{lemma:CLriskLB}
Let $M := [\mu_1\ \mu_2\ \cdots\ \mu_C] \in \RR^{d\times C}$. Then,
\begin{align}
L(f)   &\geq G(M), \nonumber \\
G(M) &:= \sum_{i, j_{1:k} \in \Ccal}
{ \left( \lambda_i  \prod_{t=1}^k \lambda_{j_t} \right) }  \psi  \big( \mu_i^\top \mu_{j_1} - 1, \dots, \mu_i^\top \mu_{j_k} - 1)  \big).
\label{eq: thuan101}   
\end{align}
The lower bound $G(M)$ can be attained if, and only if, there is within-class variance collapse, i.e., $f$ maps all samples belonging to any class, to the mean representation vector of the class, i.e., $\forall x \in \Xcal, f(x) = \mu_{y(x)}$, and $\forall i \in \Ccal, ||\mu_i||^2 = 1$.
\end{lemma}
\begin{proof}
Please see Appendix~\ref{apd:CLriskLB}.
\end{proof}
Specializing (\ref{eq: thuan101}) to the InfoNCE loss function defined in Appendix~\ref{apd: 1} we get
%
\begin{corollary}\label{cor:InfoNCElb}
For the InfoNCE loss function defined in Appendix~\ref{apd: 1}, 
\begin{align}
G(M) &= \sum_{i,j_{1:k} \in \Ccal}
{ \left( \lambda_i  \prod_{t=1}^k \lambda_{j_t} \right) } \log  \left(1+  \frac{1}{k}\sum_{t=1}^k    e^{\mu_i^\top \mu_{j_t} - 1 } \right). 
\end{align}    
\end{corollary}
%
%
In practice, $k$ could be large (e.g., $k = 128, 256, 512, \dots$). In the limit $k \rightarrow + \infty$, the expression for the lower bound in Corollary~(\ref{cor:InfoNCElb}) simplifies substantially.
\begin{corollary}
\label{cor:InfoNCElb_infinite_k}
For InfoNCE loss,  
\begin{align}  
\lim_{k\rightarrow \infty} L(f) 
&=  \EE\bigg[\log \Big(1 + \EE\Big[e^{f^\top(x)(f(x^-_1)-f(x^+))}\Big|x,x^+\Big] \Big)\bigg] \label{eq:infinitekL} \\
&\geq \lim_{k\rightarrow \infty} G(M) \nonumber \\ 
&= \sum_{i\in \Ccal} \lambda_i  \log \Big(1 +  \sum_{j \in \Ccal}  r(j|i)\, e^{\mu_i^\top \mu_j - 1} \Big). \label{eq:infinitekG}
\end{align}    
\end{corollary}
\begin{proof}
    Please see Appendix~\ref{apd:proof_InfoNCElb_infinite_k}.
\end{proof}

\section{Characterizing and computing optimal class means}\label{sec: main result}
An optimal matrix $M^* \in \RR^{d \times C}$ which minimizes the lower bound in Lemma~\ref{lemma:CLriskLB} can be found by solving the following constrained-optimization problem:
\begin{align}
\min_{M \in \Mcal}  &G(M), \quad\text{where} \label{eq: thuan30}
\end{align} 
\begin{align}
\Mcal := \{M = [\mu_1\cdots\mu_C] \in \mathbb{R}^{d \times C}: \forall i \in \Ccal, \|\mu_i\|^2 = 1\}. \label{eq:Mcaldef} 
\end{align}
A solution to (\ref{eq: thuan30}) exists since the objective function $G(M)$ is continuous and the constraint set $\Mcal$ is compact. However, neither is the objective function in (\ref{eq: thuan30}) convex with respect to $M$ nor is the constraint set defined in (\ref{eq:Mcaldef}) convex due to the unit norm equality constraint. This complicates the development of computational methods for finding an optimal solution. 
Under additional special conditions on the representations, optimal solutions can be identified. For example, if the representations are confined to the non-negative orthant of $\RR^d$, which can be implemented through the application of a non-negative activation function, e.g., ReLU, to the final layer of the neural network of the representation map, then we have the following result.
\begin{theorem} \label{theorem:non-neg} For all $f \in \Fcal$, let $f(\Xcal) \subseteq \RR^d_{\geq 0}$. Then for all $f \in \Fcal$,
\[
L(f) \geq \psi(-1,\ldots,-1)
\]
with equality, if, and only if, $d \geq C$, $\mu_{1:C}$ are orthonormal, and $\forall x \in \Xcal$, $f(x) = \mu_{y(x)}$.
\end{theorem}
\begin{proof}
Please see Appendix~\ref{apd:proof_of_theorem_non-neg}.
\end{proof}
Theorem~1 in \citep{kini2024symmetric} is a specialized version of Theorem~\ref{theorem:non-neg} for a restricted form of the InfoNCE loss. These results show that with additional non-negativity constraints on the representation and $d\geq C$, the geometry of the optimum representations is an orthonormal system \textit{irrespective of the class imbalance}. To characterize the geometry without non-negativity constraints, let
\begin{equation*}
A := M^\top M =
\begin{bmatrix}
   \mu_1^\top \mu_1 & \mu_1^\top \mu_2 & \dots & \mu_1^\top \mu_C \\
   \mu_2^\top \mu_1 & \mu_2^\top \mu_2 & \dots & \mu_2^\top \mu_C \\
   \dots & \dots & \dots & \dots \\
   \mu_C^\top \mu_1 & \mu_C^\top \mu_2 & \dots & \mu_C^\top \mu_C \\
\end{bmatrix} \in \mathbb{R}^{C \times C},
\end{equation*}
denote the Gram matrix of class means in representation space composed of their pairwise inner products. By construction, $A$ is symmetric, i.e., $A^\top = A$, and positive semi-definite (PSD), i.e., $A \succcurlyeq 0$, which means that $\forall u \in \RR^C, u^\top A u \geq 0$, and additionally, $\forall i \in \Ccal, \ A_{ii} = 1$ since $A_{ii} = ||\mu_i||^2 = 1$ is needed to attain the lower bound in Lemma~\ref{lemma:CLriskLB}. Let 
\begin{align}
\Acal^* &:= \{A \in \RR^{C\times C}: A = A^\top, A\succcurlyeq 0, \forall i\in\Ccal, A_{ii} = 1\} \text{ and } 
\label{eq:Astarcaldef} \\
S(A) &:= \sum_{i, j_{1:k} \in \Ccal}
{ \left( \lambda_i  \prod_{t=1}^k \lambda_{j_t} \right) }  \psi  \big( A_{ij_1}-1, \dots, A_{ij_k}-1 \big).
\label{eq:Sdefn} 
\end{align}
Under certain conditions, a solution to (\ref{eq: thuan30}) can be found by minimizing (\ref{eq:Sdefn}) over (\ref{eq:Astarcaldef}).
\begin{lemma}
\label{lemma: equivalent} For all $M \in \Mcal$, $ M^\top M \in \Acal^*$ and $G(M) = S(M^\top M)$. Let $A^* \in \Acal^*$ be a solution to the following optimization problem 
\begin{align}
    \min_{A \in \Acal^*} S(A). \label{eq: thuan31-old}
\end{align}
If there exists an $M^* \in \Mcal$ such that $ (M^*)^\top M^* = A^*$, then $M^*$ is solution to (\ref{eq: thuan30}). 
\end{lemma}
\begin{proof}
Please see Appendix~\ref{apd:proof_of_lemma-equivalent}.
%
\end{proof}
Lemma~\ref{lemma: equivalent} proves that if the global minimizer $A^*$ of (\ref{eq: thuan31-old}) can be factorized as $(M^*)^\top M^*$, then $M^*$ is a solution to the original objective (\ref{eq: thuan30}).
We note that any $M \in \Mcal$ can be mapped to an $A=M^\top M \in \Acal^*$. However, if $d< C$, it may not be possible to decompose all $A \in \Acal^*$ as $A = M^\top M$ for some $M \in \Mcal$.
\begin{lemma}\label{lemma:convexproblem}
The function $S(\cdot)$ is a strictly convex function over $\Acal^*$. The constraint set $\Acal^* \subset \RR^{C\times C}$ is convex and compact. Therefore, the minimization problem in (\ref{eq: thuan31-old}) is a convex optimization problem and has a unique solution $A^* \in \Acal^*$, i.e., 
\begin{align}
    S(A^*) = \min_{A\in \Acal^*} S(A). \label{eq: thuan31}
\end{align}
\end{lemma}
\begin{proof}
    Please see Appendix~\ref{apd:proof_lemma_convexproblem}.
\end{proof}
If $r = \text{rank}(A^*)$, then $r \leq C$ since $A^* \in \RR^{C\times C}$. Also note that $\text{rank}((M^*)^\top M^*) \leq \min\{d,C\}$ since $M^* \in \RR^{d\times C}$. Now, if $d \geq C$, then \textit{any} $C \times C$ PSD matrix (therefore also $A^*$) can be factorized as $(M^*)^\top M^*$ (via the eigen-decomposition of $A^*$ truncated to $r$ nonzero eigenvalues). There is no ``low-rankness'' associated with the aforementioned statement. Interestingly, the next lemma proves that the unique optimal solution to (\ref{eq: thuan31-old}) is rank-deficient. Specifically, it proves that $A^*$ is \textbf{\textit{guaranteed}} to have rank not exceeding $C-1$  even though there is no rank constraint imposed on the optimization problem ($A^* \in \RR^{C \times C}$). 
\begin{lemma}\label{lemma:rank_of_opt_Gram_mtx}
The unique solution $A^* \in \Acal^*$ to (\ref{eq: thuan31-old}) has $\text{rank}(A^*) =: r \leq C-1$. Therefore, the minimum eigenvalue of $A^*$ is zero.
\end{lemma}
\begin{proof}
    Please see Appendix~\ref{apd:proof_rank_of_opt_Gram_mtx}.
\end{proof}
We note that the result of Lemma~\ref{lemma:rank_of_opt_Gram_mtx} is a \textit{\textbf{consequence}} of the uniqueness of the optimal $A^*$ proved in Lemma~\ref{lemma:convexproblem}. It is \textit{\textbf{not}} a low-rank assumption or constraint. The next theorem puts the implications of Lemmas~\ref{lemma: equivalent}--\ref{lemma:rank_of_opt_Gram_mtx} together and proves that as long as $d \geq C-1$, it is possible to factorize $A^*$ as $(M^*)^\top M^*$ and it explicitly constructs $M^*$ using $A^*$'s eigen-decomposition truncated to $r$ nonzero eigenvalues.

\begin{theorem}[Optimal Class Means] \label{theorem: 1}
Let $A^* = U_r \Sigma_r U_r^\top$ be the unique solution to \ref{eq: thuan31-old}, where $r := \text{rank}(A^*) \leq C-1$, $\Sigma_r \in \RR^{r\times r}$ is a diagonal matrix with the $r$ strictly positive eigenvalues of $A^*$ along the main diagonal, and $U_r \in \RR^{C\times r}$ is the matrix of $r$ orthonormal eigenvectors of $A^*$ corresponding to the $r$ positive eigenvalues.
If $d \geq C - 1$, then $(M^*)^\top := [ U_r \sqrt{\Sigma_r}\ \ \ 0_{\text{\tiny{\(C \times d-r+1\)}}}]$  is a solution to (\ref{eq: thuan30}), where $\sqrt{\Sigma_r}$ is a diagonal matrix with the square roots of the $r$ positive eigenvalues of $A$ along the main diagonal and $0_{\text{\tiny{\(C \times d-r+1\)}}}$ is the $C \times d-r+1$ matrix of all zeros.\footnote{If $d=r-1$, then $0_{\text{\tiny{\(C \times d-r+1\)}}}$ is void.} Moreover, $\forall i \in \Ccal, \|\mu_i^*\|^2=1$ where $\mu_i^*$ ($i^{\text{th}}$ column of $M^*$) is an optimal class mean vector in representation space for class $i$. 
\end{theorem}
\begin{proof}
Please see Appendix~\ref{apd:proof_of_theorem-opt_class_means}
\end{proof}

The solution $A^*$ to (\ref{eq: thuan31-old}) (the optimum Gram matrix) is unique. However, the solution to (\ref{eq: thuan30}) is not unique due to the rotational invariance of the loss function. The $M^*$ defined in Theorem~\ref{theorem: 1} is just one solution to (\ref{eq: thuan30}) when $d \geq C-1$. Still, when $d \geq C-1$, any solution $\hat{M}^*$ to (\ref{eq: thuan30}) will also satisfy $(\hat{M}^*)^\top {\hat{M}^*} = A^*$ because $G(\hat{M}^*) = S((\hat{M}^*)^\top {\hat{M}^*}) \geq S((M^*)^\top M^*) = S(A^*)$ and $A^*$ is the unique minimizer of $S(A)$ over $\Acal^*$. 

\begin{remark}
\label{remark: 2}
For $d \geq C-1$, Theorem \ref{theorem: 1} offers a way to find the optimal mean vectors $\mu_1^*, \mu_2^*, \dots, \mu_C^*$ via convex optimization. In our simulations in Section~\ref{sec: numerical results}, we utilize the convex optimization package CVX~\cite{grant2014cvx} to compute $A^*$ and then use the spectral decomposition in Theorem~\ref{theorem: 1} to compute an optimal mean representation vector matrix $M^*$. 
\end{remark}
\begin{corollary}
\label{cor:varcollapse}
Let $d \geq C-1$, ${M^*} = [\mu_1^*, \mu_2^*, \dots, \mu_C^*] \in \Mcal$ be a solution to (\ref{eq: thuan30}), and $A^* = (M^*)^\top M^*$ be the unique solution to (\ref{eq: thuan31-old}). Then $L(f) = G(M^*) = S(A^*)$ for an $f \in \Fcal$, if, and only if, $\forall x \in \Xcal, f(x) = \mu^*_{y(x)}$. 
\end{corollary}
\begin{proof}
This follows immediately from the optimality of $A^*$ and $M^*$ and Lemma~\ref{lemma:CLriskLB}.
\phantom{\hfill}
\end{proof}
The condition $C-1 \leq d$ is an assumption on the number of classes (an intrinsic property of dataset or application) relative to the representation dimension (a design choice, e.g., via suitable neural net architecture). This condition is also required in many papers to show the NC phenomenon and the existence of the ETF-structure, e.g., \cite{jiang2024hard, graf2021dissecting, wang2023towards, pmlr-v202-dang23b}. But they are all in the setting where classes are balanced, i.e., $\forall i \in \Ccal,\, \lambda_i = 1/C$. In practice, $C - 1 \leq d$ in applications where the number of classes is much smaller than the dimension of the representation space, \textit{e.g.,} $d = 512$  in ResNet-18 compared to $C = 10$ in the CIFAR10 dataset and $C = 100$ in the CIFAR100 dataset. The case $d < C-1$, e.g., in LLMs, is currently an unresolved open problem.

An interesting implication of Corllary~\ref{cor:varcollapse} is that, in order to globally minimize the contrastive risk, we only require the dimension of the representation space to be $d=C-1$. This suggests that current approaches which use a very high-dimensional representation space to learn the features, may be inefficient in terms of storage and computational resources.

\section{Equiangular properties of optimal class means}\label{sec: equiangular}
In this section, we show that the optimal class of classes that are equiprobable have an equiangular geometric structure. These are consequences of the uniqueness of $A^*$.
\begin{theorem}
\label{theorem:equiangular-properties}
Suppose that there are two distinct classes $i$ and $j$ with the same probability, i.e.,  $\lambda_i=\lambda_j$. Let ${M^*}=[\mu_1^*, \mu_2^*, \dots, \mu_C^*]$ be an optimal mean vector matrix such that ${M^*}^\top {M^*} = A^*$. Then, 
\[
\forall n \in \Ccal\setminus\{i, j\}, \, {\mu_i^*}^\top \mu_n^* = {\mu_{j}^*}^\top \mu_{n}^*.
\]
\end{theorem}
\begin{proof}
The key idea of the proof is to show that if we swap $\mu_i^*$ and $\mu_j^*$ in $M^*$ to form a new matrix $Q$, then $S(Q^\top Q) = S({M^*}^\top M^*)$. 
The detailed proof is presented in Appendix~\ref{apd:proof_of_theorem_equiangular-properties}.
\end{proof}
The following Corollary expands the results of Theorem~\ref{theorem:equiangular-properties} to the scenario where multiple classes have the same probability. 
\begin{corollary}
\label{cor: 1}
Let $\Ccal := \{1,2,\dots, C\}$ denote the set of $C$ classes, and $\Ccal' \subseteq \Ccal $ a subset of classes that have the same probability. Then, 
\[
\forall i,j, \in \Ccal', i \neq j,\,{\mu_i^*}^\top{\mu_j^*} = \text{constant}.  
\]
\end{corollary}
\begin{proof}
Please see Appendix~\ref{apd:proof_of_corollary-1}.
\end{proof}
\begin{corollary}\label{cor:ETF}
If all classes are equiprobable, \textit{i.e.}, $\Ccal' = \Ccal$ in Corollary~\ref{cor: 1}, then for all $i, j \in \Ccal, i\neq j$, we have ${\mu_i^*}^\top \mu_j^* = -1/(C-1)$, $ \forall i \in \Ccal,\,\|\mu_i^*\|^2=1$, and $\sum_{i\in\Ccal} \mu^*_i = 0$, i.e., the optimal class means form an equiangular tight frame (ETF) in $\RR^d$.
\end{corollary}
\begin{proof}
    Please see Appendix~\ref{apd:proof_ETF}.
\end{proof}
\begin{remark}
\label{remark: 5}
Corollary~\ref{cor:ETF} resolves a question that was left open in \cite{jiang2024hard} for balanced datasets and the general \CL{} loss function $\psi$, namely whether the ETF geometry is optimal when the positive pairs are not conditionally independent given their class label and the classes of the positive and negative samples can collide. 
\end{remark}

We note that there is no simple analytical closed-form expression available for the angles between the optimal mean vectors in the general imbalanced setting. They can, however, be computed via a convex program as we noted in Remark~\ref{remark: 2}.

\section{Minority collapse}\label{sec: minority collapse}
Minority collapse is a phenomenon that can be observed in imbalanced datasets. It refers to a scenario where the representations of all the samples in several distinct minority classes (classes with small probabilities) collapse into a single vector. In deep \textit{classifier} neural networks it is known that minority collapse will occur if the class imbalance is extreme \citep{fang2021exploring, pmlr-v202-dang23b,dang2024neural, hong2024neural}. In this section, we show that minority collapse also occurs in contrastive learning for imbalanced datasets. To formally demonstrate the existence of this phenomenon, we consider the special scenario where $1 > \lambda_1 > \lambda_2=\lambda_3=\ldots=\lambda_C = \tfrac{1-\lambda_1}{C-1}>0$, \textit{i.e.}, the first class is the majority class and the remaining $C-1$ classes are minority classes. This special scenario is motivated by considerations of analytical tractability and the goal of deriving an explicit non-asymptotic sufficient condition under which the minority collapse phenomenon is guaranteed to manifest. We will prove that if the probability of the minority classes $\tfrac{1-\lambda_1}{C-1}$ is less than a certain threshold, or equivalently if $\lambda_1$ is greater than a threshold, then minority collapse will occur. We will derive an explicit formula for this threshold in terms of $C, k$, and bounds on the subgradients of the loss function $\psi$. We will then apply the formula to the InfoNCE loss function and derive a numerical threshold that holds for all $C\geq 3$ and all $k$.
\begin{theorem}[Sufficient conditions for minority collapse] \label{theorem:minoritycollapse}
Let $C\geq 3$ and $1 > \lambda_1 > \lambda_2=\lambda_3=\ldots=\lambda_C = \tfrac{1-\lambda_1}{C-1}>0$. 
Let $S(\cdot)$ be as in (\ref{eq:Sdefn}) with $\psi:\RR^k \longrightarrow \RR$ be strictly convex and argument-wise strictly increasing. 
Then  $\psi$ is Lipschitz over $\Vcal := [-2,0]^k$ with a Lipschitz constant $\Delta_2 < \infty$.
For all $u \in \RR^k_{\geq 0}\setminus\{\bfzero\}$ and all $t \in [-2,0]$, let $\phi_u(t) := \psi(t\,u)$. Then, 
\begin{align}
\forall\, u \in \RR^k_{\geq 0}\setminus\{\bfzero\},\,\exists\,\delta_u \in (0, \infty)\,:\, \forall\, t, t' \in [-2,0],\, t'\leq t,\quad  (t - t')\,\delta_u \leq \phi_u(t) - \phi_u(t'). \nonumber   
\end{align}
Let $\delta_{\bfzero} := 0$ and
\begin{align}
\delta_* := \min_{u \in \{0,1\}^k \setminus\{\bfzero\}}  \delta_u \in (0, \infty).\label{eq:mindelta}    
\end{align}
For all $y^-_{1:k} \in \Ccal^k$, let $u({y,y^-_{1:k}}) := (1(y^-_1 \neq y),\dots,1(y^-_k \neq y))^\top \in \{0,1\}^k$, where $1(\cdot)$ is the indicator function. With $(y, y^-_{1:k})$ distributed as in (\ref{eq:joint1of2}), if
\begin{align}
\lambda_1 \geq \frac{1}{1+ \frac{1}{\gamma_C \sqrt{k} \Delta_2}\,\EE[\delta_{u(y,y^-_{1:k})}| \Ecal_1]},
\label{eq:MClambda1suffcond}    
\end{align}
where $\gamma_C := \tfrac{2(C-1)}{(C-2)}$, then for all $i \in \Ccal\setminus\{1\}$, $\mu^*_i = -\mu^*_1$ with $||\mu^*_1||=1$, i.e., we have minority collapse. The sufficient condition for minority collapse given by (\ref{eq:MClambda1suffcond}) is satisfied if 
\begin{align}
\lambda_1 \in [\tau, 1), \quad 
\tau :=  \tfrac{1}{1+ \frac{\delta_*}{\gamma_C\,\sqrt{k}\,\Delta_2}} \in (0,1). 
\label{eq:lambda1threshold}    
\end{align}
\end{theorem}
\begin{proof} The detailed proof is long and presented in Appendix~\ref{apd:proof_minoritycollapse}. It consists of the following steps. Using Theorem~\ref{theorem:equiangular-properties}, Corollary~\ref{cor: 1}, the given class proportions, and the rank deficiency of $A^*$ proved in Lemma~\ref{lemma:rank_of_opt_Gram_mtx}, we first show (see Lemma~\ref{lemma:Astarform} in Appendix~\ref{apd:Astarform}) that $A^*$ belongs to a family of matrices parameterized by a single scalar $a \in [-1,1]$ which equals the inner product between $\mu^*_1$ and $\mu^*_i$ for any class $i\neq 1$. Next, using standard results in convex optimization theory, the fact that $\psi$ is argument-wise strictly increasing, and the definition of subgradients and subdifferentials, we show (see Lemma~\ref{lemma:subgradients} in Appendix~\ref{apd:subgradients}) that $\psi$ is Lipschitz-$\Delta_2$ over $[-2,0]^k$ and also establish the properties of $\phi_u$ stated in the theorem. We also prove that $A^*(a)$ is element-wise Lipschtiz-$\gamma_C$ (Lemma~\ref{lemma:Astarbound} in Appendix~\ref{apd:proof_Astarbound}). By combining these results, in Appendix~\ref{apd:proof_minoritycollapse} we prove that if condition (\ref{eq:MClambda1suffcond}) is satisfied, then
$S(A^*(a))$ is a strictly increasing function and therefore minimized at $a=-1$ which implies that for all $i \in \Ccal\setminus\{1\}$, $\mu^*_i = -\mu^*_1$, i.e, we have minority collapse. Finally, we also show that the sufficient condition for minority collapse given by (\ref{eq:MClambda1suffcond}) is satisfied if condition (\ref{eq:lambda1threshold}) is satisfied.
\end{proof}
We note that the condition $\lambda_1 \in \left( \frac{1}{1 + \delta_*/ (\gamma_C\,\sqrt{k}\,\Delta_2)}, 1\right)$ is sufficient, but not necessary, for minority collapse and the threshold $\tau = \frac{1}{1 + \delta_*/ (\gamma_C\,\sqrt{k}\,\Delta_2)}$ may be quite loose because it is based on $\delta_*$, the smallest value of $\delta_u$ among all $u \neq \bfzero$. Moreover, $\tau$ may depend on $k$ and may go to $1$ as $k$ increases to infinity. For specific loss functions, such as InfoNCE, a more careful analysis of (\ref{eq:MClambda1suffcond}) can yield a non-trivial threshold that is independent of $k$. This is illustrated in the following corollary. 
\begin{corollary}\label{cor:InfoNCElambda1threshold}
For the InfoNCE loss function defined in Appendix~\ref{apd: 1},
condition (\ref{eq:MClambda1suffcond}) for minority collapse in Theorem~\ref{theorem:minoritycollapse} is satisfied if
\[
\lambda_1 \in [\tau_C,1), \text{ where } \tau_C := \frac{1-\sqrt{1-\beta_C^2}}{\beta_C} \text{ and } \beta_C := \frac{1}{1 + \tfrac{1}{4\gamma_C(1+3e^2)}}.
\label{eq:InfoNCElambda1threshold} 
\]
Moreover, for all $C \geq 3$, $\tau_C \leq \tau_3 \approx 0.9292$. Thus, $\lambda_1 \geq 0.9292$ is a sufficient condition for minority collapse for the InfoNCE loss function, irrespective of the number of classes $C$ or the number of negative samples per anchor sample $k$.
\end{corollary}
\begin{proof}
    Please see Appendix~\ref{apd:proof_InfoNCElambda1threshold}.
\end{proof}
This completes the development of all our theoretical results for the \UCL{} setting. 
\begin{remark}
As mentioned in Section~\ref{sec: problem setup}, all our theoretical results in Sections~\ref{sec: lower bound} -- \ref{sec: minority collapse} continue to hold, with minor adjustments to some expressions, in the \SCL{} setting as well. This is discussed in detail in Appendix~\ref{apd:extensions_to_SCL}. Numerical results that corroborate and illustrate the theoretical results are presented in Section~\ref{sec: numerical results}.
\end{remark}

\section{Computer experiments}\label{sec: numerical results}
This section provides two different types of experiments to verify the two phenomena investigated in Section~\ref{sec: main result} and Section~\ref{sec: minority collapse}, namely, \textbf{(1) intra-class variance-collapse (Section~\ref{sec: Intra-class variance-collapse}):}  the representations of all the samples from the same class collapse to their class mean vector, and the optimal class mean vectors can be computed via a convex-optimization program and \textbf{(2) minority-collapse (Section~\ref{sec: minority collapse experiment}):} if the probabilities of the minor classes are less than a threshold, then not only do the representations of all samples in the minor classes collapse to their class means, but also their class means collapse into a single vector. Since methods to select negative samples differ in the supervised (\SCL) and unsupervised (\UCL) settings, each experiment is performed under two different setups: \textbf{(a) \SCL: the negative samples are selected from a class that is different from that of the positive samples}, and \textbf{(b) \UCL: the negative samples are selected from the whole dataset, which may include the class of positive samples}. Although all our theoretical results are for a general loss function, we focus on the well-known InfoNCE loss for the experiments.

Since practical implementations use mini-batching, we now describe the mini-batch construction and the batch loss calculation used in our experiments. Let $\Xcal_{\text{batch}} := \{x_{1:N}\}$ be a mini-batch (potentially a multiset) of $N$ samples. For a given anchor sample $x_i \in \Xcal_{\text{batch}}$, and a positive integer $k$, let $\Xcal^-_{\text{batch},x_i}
:= \{x^-_{i1},\ldots,x^-_{ik}\}$ be a multiset of $k$ negative samples sampled from $\Xcal_{\text{batch}}$ \textit{with} replacement. In the \UCL{} setting where the negative samples can be selected from any classes in the dataset, $x^-_{i1},\ldots,x^-_{ik}$ are selected uniformly at random from $\Xcal_{\text{batch}}$. In the \SCL{} setting where negative samples must be selected from classes other than that of the positive samples, $x^-_{i1},\ldots,x^-_{ik}$ are selected uniformly at random from classes different from that of $x_i$. Let $\Xcal^+_{\text{batch},x_i}$ denote the set of samples in the batch with same label as sample $x_i$, \textit{i.e.}, $\Xcal^+_{\text{batch},x_i} := \{x\in \Xcal_{\text{batch}} : y(x) = y(x_i)\}$.  Then, our implemented loss function over a batch is:
{
\begin{align}
&\mathcal{L}_{\textup{batch}} \nonumber \\
&= \frac{1}{N}
\sum_{i=1}^{N}
\left[ \frac{1}{|\Xcal^+_{\text{batch},x_i}|}\! 
\sum_{x_j \in \Xcal^+_{\text{batch},x_i}}\hspace{-3ex} \psi_{\text{InfoNCE}}(f(x_i)^\top (f(x^-_{i1}) - f(x_j)), \ldots, f(x_i)^\top(f(x^-_{ik}) - f(x_j))  \right] \nonumber\\
&=  \frac{1}{N}
\sum_{i=1}^{N}
\left[ \frac{1}{|\Xcal^+_{\text{batch},x_i}|} \sum_{x_j \in \Xcal^+_{\text{batch},x_i}}\hspace{-2ex} \log\bigg(1 + \frac{1}{k} \sum_{q=1}^k e^{f(x_i)^\top(f(x^-_{iq}) - f(x_j))}\bigg)   \right]
\end{align}
}
where $|\Xcal^+_{\text{batch},x_i}|$ denotes the size of set $\Xcal^+_{\text{batch},x_i}$ and  $\psi_{\text{InfoNCE}}$ is the InfoNCE loss function defined in (\ref{eq:InfoNCElossfunction}). 

Thus, the batch loss is computed by an outer average and an inner average of the loss function $\psi_{\text{InfoNCE}}(\cdot)$. In the outer average, $\psi_{\text{InfoNCE}}(\cdot)$ is averaged across $N$, $(k+1)$-tuples of anchor and negatives $(x_i,x^-_{i1},\ldots,x^-_{ik})$, where we first select the anchor from $\Xcal_{\text{batch}}$ and then $k$ negatives associated with $x_i$ from $\Xcal_{\text{batch}}$ according to the appropriate negative sampling distribution of the \SCL{} or \UCL{} setting. For a given $(k+1)$-tuple $(x_i,x^-_{i1},\ldots,x^-_{ik})$, in the inner average,  $\psi_{\text{InfoNCE}}(\cdot)$ is averaged across   $|\Xcal^+_{\text{batch},x_i}|$ positive samples that have the same label as the anchor $x_i$. The batch loss can be interpreted as an empirical instantiation of the population loss via the nested (iterated) expectation 
\begin{multline*}
\EE[\psi_{\text{InfoNCE}}(f(x)^\top (f(x^-_{1}) - f(x^+)), \ldots, f(x)^\top(f(x^-_{k}) - f(x^+))] 
\\ = \EE\Big[\EE\big[\psi_{\text{InfoNCE}}(f(x)^\top (f(x^-_{1}) - f(x^+)), \ldots, f(x)^\top(f(x^-_{k}) - f(x^+))| {x,x^-_{1:k}}\big]\Big] 
\end{multline*}
where the inner expectation is over $x^+$ for a given $(k+1)$-tuple $(x,x^-_1,\ldots,x^-_k)$.  The overall loss in an epoch is the average of the batch loss across all the mini-batches in that epoch.

All our theoretical results were established for the batch setting.  To ensure that they also hold in the mini-batch setting, as discussed in the recent work of \cite{kini2024symmetric}, mini-batches must be carefully constructed to prevent the formation of disjoint groups of non-interacting samples that remain ``frozen'' across epochs. One method to prevent this, proposed in \cite{kini2024symmetric}, is the so-called \textit{batch-shuffling} method where the samples are divided into mini-batch partitions, with a random reshuffling of all samples in every epoch. We adopt this batch-shuffling method in our experiments.

\subsection{Intra-class variance-collapse}
\label{sec: Intra-class variance-collapse}
In this section, we provide the numerical results to verify the intra-class variance-collapse phenomenon. We used a dataset comprising three classes extracted from the CIFAR-10 dataset. Specifically, we selected the first 1500, 750, and 750 image samples, respectively, from the first three classes (\textit{i.e.}, $C=3$), namely bird, automobile, and airplane, of the CIFAR10 dataset to form our dataset comprising 3000 samples. This corresponds to $\lambda_1 = 0.5$ and $\lambda_2 = \lambda_2 = 0.25$. We utilized the ResNet-50 architecture to implement the representation function $f$. To satisfy the condition $C = 3 \leq d + 1$ in Theorem \ref{theorem: 1}, we set the dimension of the representation space to $d=2$. We set the batch size and the number of epochs to 512 and 1000, respectively,  and the number of negative samples to $k=512$. We optimized the empirical \CL{} risk using the Adam optimizer with a learning rate of $0.001$. 

\begin{figure}[!ht]
\begin{tabular}{ccc}
    \includegraphics[width= 0.3 \linewidth]{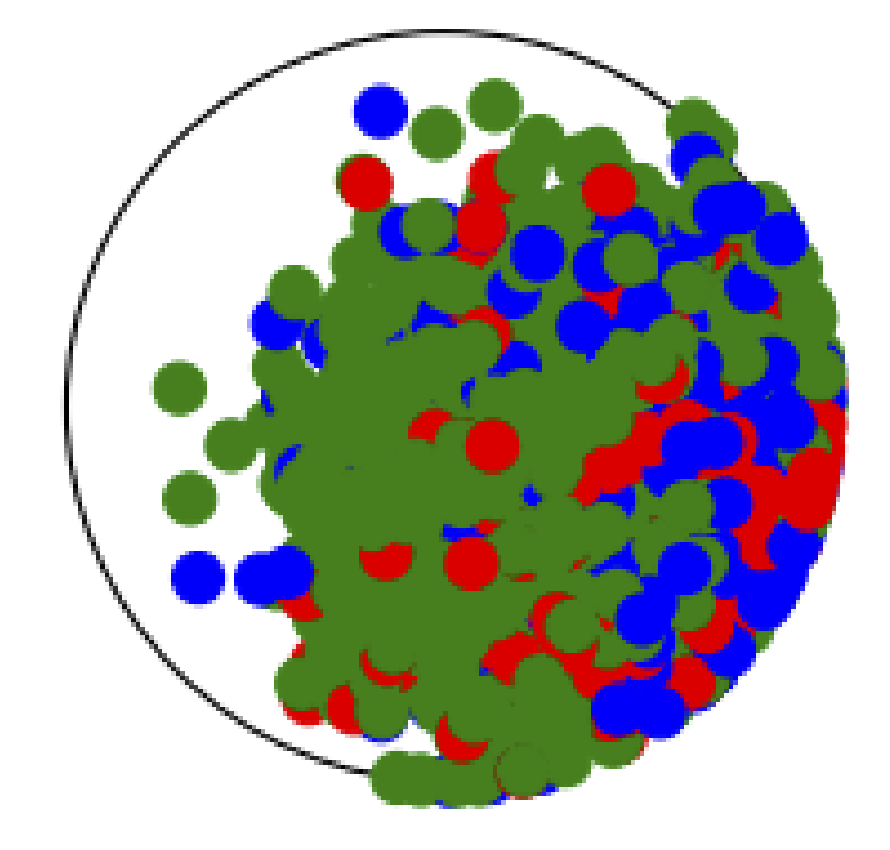} &
    \includegraphics[width= 0.28 \linewidth]{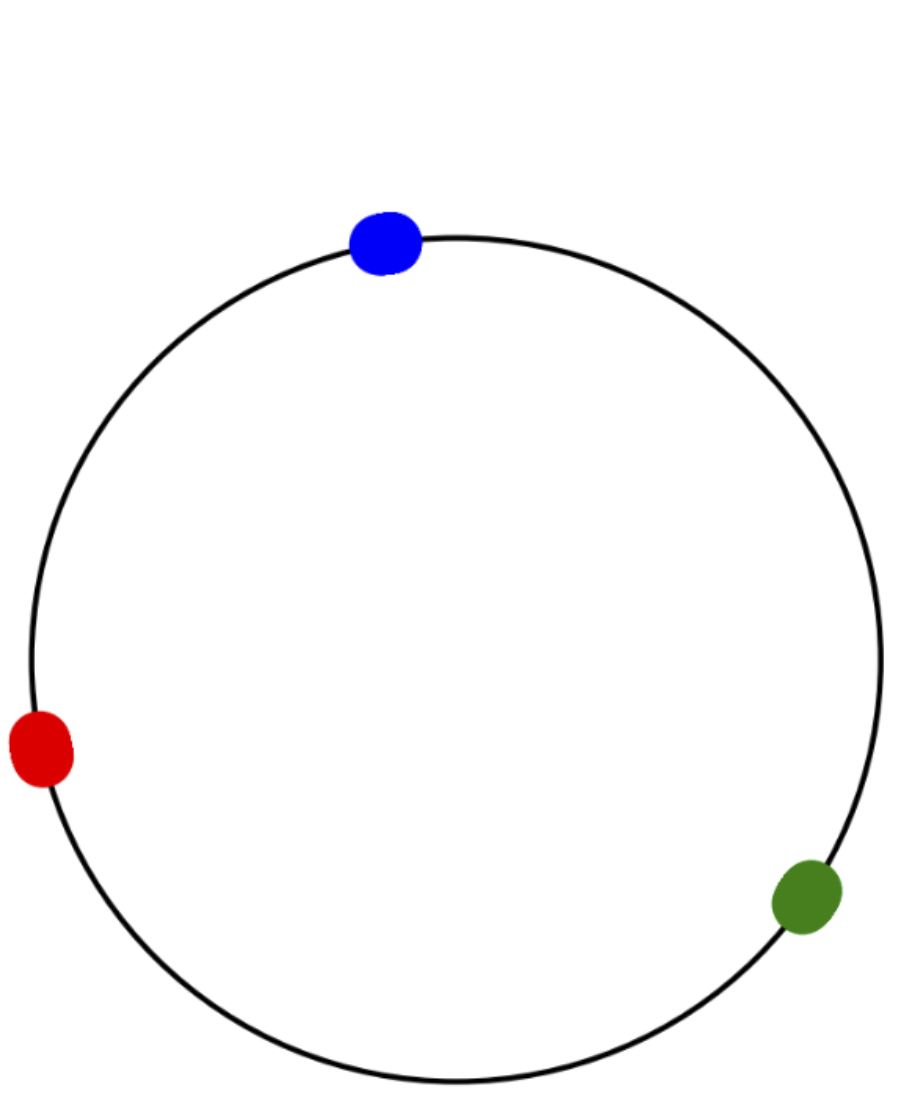} &
    \includegraphics[width= 0.28 \linewidth]{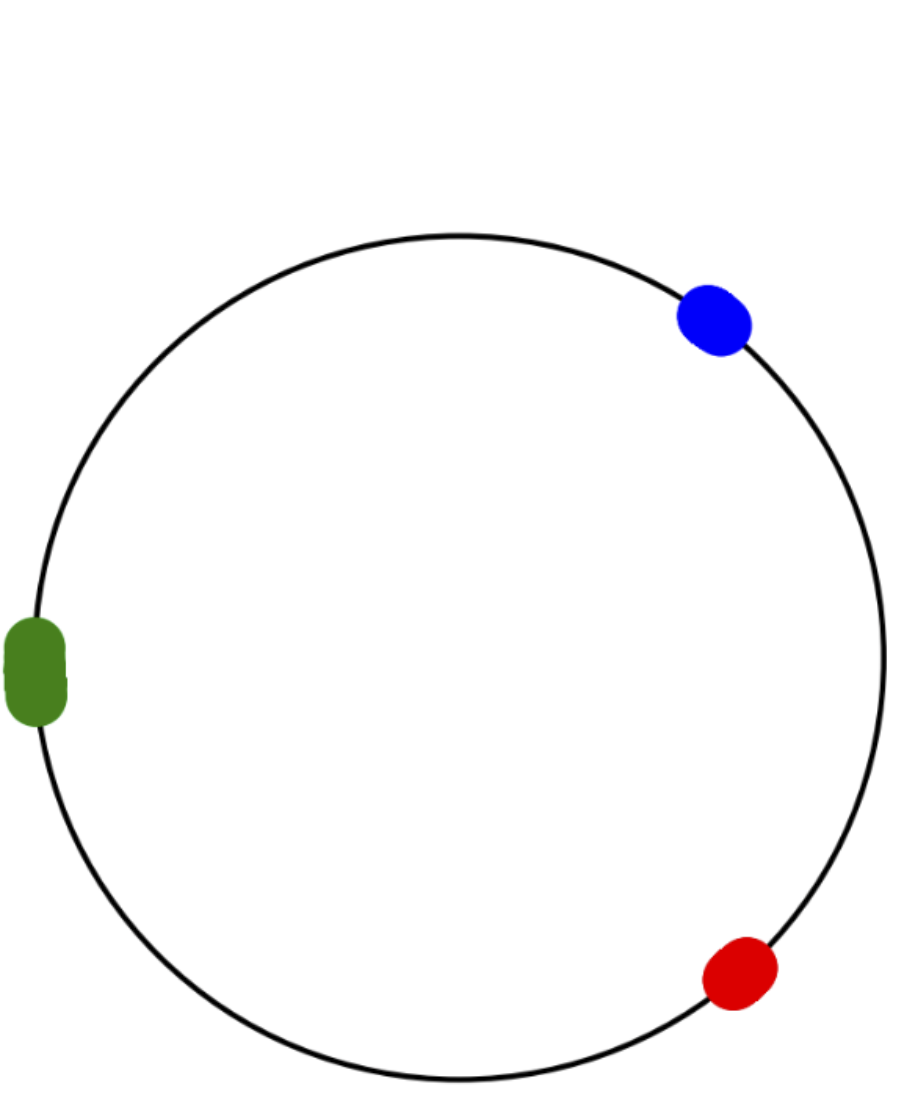} \\
   Before training &  \SCL{} post-training & \UCL{} post-training 
\end{tabular}    
  \caption{Intra-class variance-collapse in imbalanced datasets. Left: $\RR^2$-space representations of 3000 images from $3$ classes (indicated by color) of the CIFAR10 dataset using the initial representation function (i.e., before training). Middle: representation vectors of the same images at the conclusion of training when negative samples, within a mini-batch, are selected from classes that are different from those of the positive samples (\SCL{} setting). Right:  representation vectors of the same images at the conclusion of training when the negative samples can be selected from the entire mini-batch (\UCL{} setting).
  }
  \label{fig: 2}%
  \vspace{10 pt}
\end{figure}

Figure~\ref{fig: 2} illustrates the two-dimensional representations of samples from three classes using: (a) the initial mapping before the commencement of training, (b) the optimal mapping at the conclusion of training in the \SCL{} setting, and (c) the optimal mapping at the conclusion of training in the \UCL{} setting. Evidently, all the samples from the same class (represented by the same color) nearly collapse to the same point, which is their class mean. As seen, when the negative samples can be selected from any classes in the dataset (\UCL{} setting), including the class of positive samples, the distance between the two minority classes (red and blue) is much smaller compared to the setup where the negative samples are selected from classes that are different from those of the positive samples (\SCL{} setting).

To verify that the optimal solutions obtained by the neural network are consistent with our theoretical results, we used the CVX modeling system \citep{grant2014cvx} to solve the convex optimization problem in (\ref{eq: thuan31}).  From Theorem~\ref{theorem: 1}, we know that the optimal mean vector matrix $M^*$ is not unique, but the optimal Gram matrix $A^*$ is unique and can be computed as the solution to a convex optimization problem. Therefore, we compare the optimal Gram matrix provided by the neural network with the one computed using CVX. The optimal Gram matrices $A^*$ obtained by the neural network and the CVX package are

\begin{tabular}{rcc}
  \\
  & \SCL{} setting & \UCL{} setting \\ 
  $A^*_{\textup{Neural-Network}}=$ 
  & $\begin{bmatrix}
    1.0000 & -0.6884 & -0.6910\\
    -0.6884 &  1.0000 &  -0.0485 \\
    -0.6910 & -0.0485 &  1.0000
    \end{bmatrix},$ 
  & $\begin{bmatrix}
    1.0000 & -0.6301 & -0.6271 \\
    -0.6301 &  1.0000 & -0.2097 \\
    -0.6271 & -0.2097 &  1.0000
    \end{bmatrix}$ \\  \\ 
  $A^*_{\textup{CVX-package}}=$ 
  & $\begin{bmatrix}
    1.0000 & -0.6889 & -0.6889  \\
    -0.6889 & 1.0000 & -0.0480     \\
    -0.6889 &  -0.0480 & 1.0000 
    \end{bmatrix},$ 
  & $\begin{bmatrix}
    1.0000 &  -0.6284 & -0.6284 \\
    -0.6284 & 1.0000  & -0.2105 \\
    -0.6284 & -0.2105 &  1.0000
    \end{bmatrix}$ \\ \\
\end{tabular}

Evidently, both the neural network and CVX optimal solutions are very similar, and this empirically validates our theoretical results.

We also note that if the classes were balanced, then from Theorem~2 in \cite{jiang2024hard}, the three optimal class means would form an equilateral triangle in the representation space (an equilateral triangle is an ETF in 2-D space). For our imbalanced datasets, the three class means clearly do not form an equilateral triangle. They do, however, form an isosceles triangle, and this empirically validates the result of Theorem~\ref{theorem:equiangular-properties} (since $\lambda_2=\lambda_3$ in this experiment). This empirically confirms our claim that ETF is not the optimal geometric structure for imbalanced classes.

\subsection{Minority collapse}
\label{sec: minority collapse experiment}

In this section, we provide the numerical results to verify the minority-collapse phenomenon. To do so, we constructed a three-class dataset with 2700, 150, and 150 image samples from the first three classes of the CIFAR-10 dataset, respectively, to form our second dataset of 3000 samples. This setup makes $\lambda_1=0.9$ and $\lambda_2=\lambda_3=0.05$, which is the case when the data is heavily imbalanced. We utilized the ResNet-50 architecture to implement the representation function $f$. Similarly to the setup in Section~\ref{sec: Intra-class variance-collapse}, to satisfy the condition $C = 3 \leq d + 1$ in Theorem \ref{theorem: 1}, we set the dimension of the representation space to $d=2$. We also set the batch size and the number of epochs to 512 and 1000, respectively, and the number of negative samples to $k=512$. We optimized the empirical \CL{} risk using the Adam optimizer with a learning rate of $0.001$. 

Figure~\ref{fig: 3} shows the representation vectors of all $3000$ samples in the dataset at the beginning and at the end of training. Evidently, the representations of the two minor classes (blue and red) have collapsed (or nearly collapsed) into one vector (shown in red color), and the representations of these two classes are diametrically opposite on the unit circle to the representations of the major class (shown in green color). These results empirically validate the main conclusions of Section~\ref{sec: minority collapse}. We further note that $\lambda_1 = 0.9$ in this experiment is below the threshold of $0.9239$ in Corollary~\ref{cor:InfoNCElambda1threshold} for UCL and  $0.9438$ in Corollary~\ref{cor:SCLInfoNCElambda1threshold} for SCL, which guarantee minority collapse. This empirically bolsters our remarks before Corollary~\ref{cor:InfoNCElambda1threshold} that the threshold for minority collapse in Theorem~\ref{theorem:minoritycollapse} is sufficient for minority collapse, but may not be necessary.

\begin{figure}[!ht]
\begin{tabular}{ccc}
    \includegraphics[width= 0.3 \linewidth]{3000_samples_initialization.png} &
    \includegraphics[width= 0.3 \linewidth]{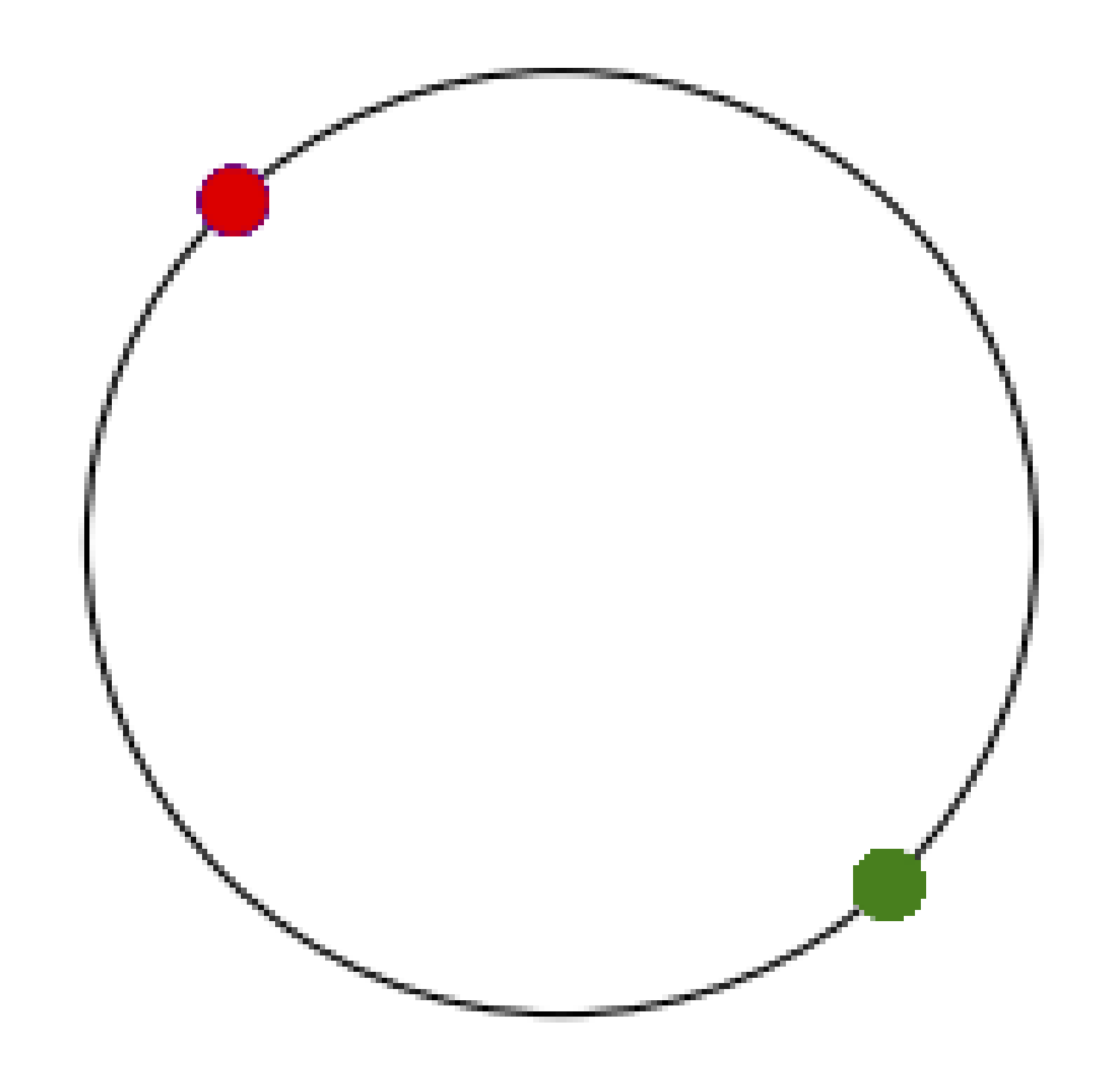} &
    \includegraphics[width= 0.28 \linewidth]{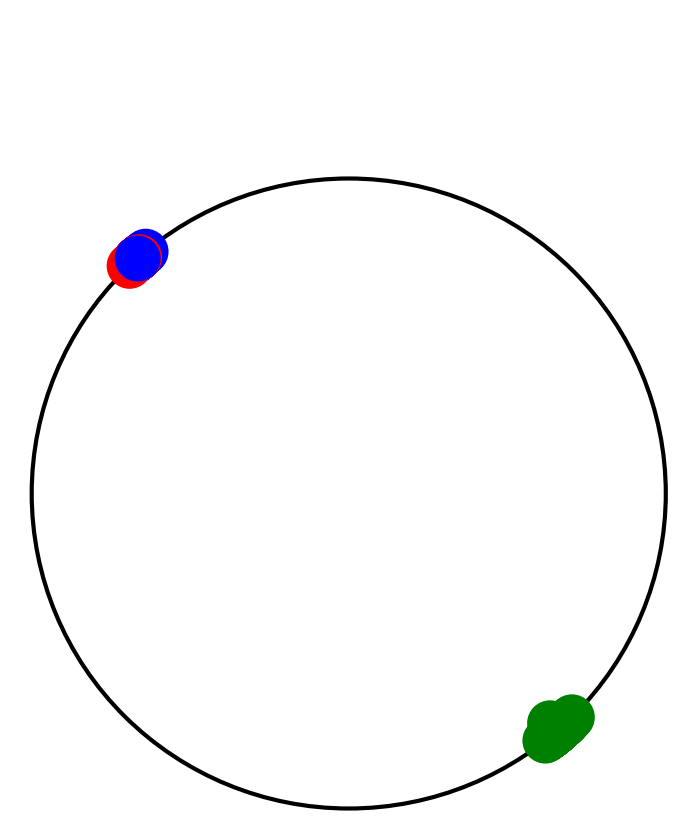}  \\
   Before training &  \SCL{} post-training & \UCL{} post-training 
\end{tabular}  
\caption{Minority collapse in heavily imbalanced datasets. The $\RR^2$-space representations of 3000 images from the first three classes of the CIFAR10 dataset before training (left sub-figure) collapse after training to two diametrically opposite points on the unit circle (middle and right sub-figures). 
The coincident or nearly coincident red/blue points in the middle and right sub-figures represent 300 samples from the second and third (minority) classes combined, whereas the green points represent the 2700 samples from the first (majority) class.}%
    \label{fig: 3}%
    \vspace{10 pt}
\end{figure}

The optimal Gram matrices $A^*$ obtained by the neural network and the CVX package are 

\begin{tabular}{rcc}
  \\
  & \SCL{} setting & \UCL{} setting \\ 
  $A^*_{\textup{minority-collapse}}=$ 
  & $\begin{bmatrix}
      1.0000 & -1.0000 & -1.0000 \\
     -1.0000 &  1.0000 &  1.0000 \\
     -1.0000 &  1.0000 &  1.0000 \\
    \end{bmatrix},$ 
  & $\begin{bmatrix}
      1.0000 & -0.9997 & -0.9997 \\
     -0.9997 &  1.0000 &  0.9999 \\
     -0.9997 &  0.9999 &  1.0000 \\
    \end{bmatrix}$ \\  \\ 
  $A^*_{\textup{CVX-package}}=$ 
  & $\begin{bmatrix}
      1.0000 & -1.0000 & -1.0000 \\
     -1.0000 &  1.0000 &  1.0000 \\
     -1.0000 &  1.0000 &  1.0000 \\
    \end{bmatrix},$ 
  & $\begin{bmatrix}
      1.0000 & -1.0000 & -1.0000 \\
     -1.0000 &  1.0000 &  1.0000 \\
     -1.0000 &  1.0000 &  1.0000\\
    \end{bmatrix}$ \\ \\
\end{tabular}

respectively, and they are identical up to the displayed numerical precision. This empirically corroborates our theoretical results that the optimal Gram matrix can be found efficiently using convex optimization. 

\section{Summary and Open Problems}
\label{sec: conclusion}
In this paper, we proved that for a general family of \CL{} losses (including the widely used InfoNCE loss) which are based on loss functions which are strictly convex and argument-wise strictly increasing, the optimal representations, 
will exhibit the intra-class variance-collapse phenomenon (representations of all samples from the same class must collapse to their class mean when globally minimizing the risk). 

Even though there is no specific optimal structure or closed-form expression available for the optimal class means in the general imbalanced case, we derived an efficient method based on convex optimization to compute these optimal class means. We also established some equiangular properties of the optimal class means of equiprobable classes.

We further investigated a special case of extreme class imbalance and showed that \CL{} also exhibits a phenomenon called minority collapse, wherein the optimal representations of all samples from the minority classes (classes with small probabilities) collapse into a single vector. Our key theoretical results were empirically validated through computer experiments.

Our work opens up several new problems that are of practical importance: (a) investigating the optimal geometry of neural collapse when the number of classes is more than the dimension of the representation space plus one -- this scenario is particularly relevant to many large language models where embedding dimensions are typically on the order of hundreds and the the number of classes range in thousands during pre-training, (b) analyzing the neural collapse phenomenon with hard-negative samples -- this is relevant to CL{} since it has been shown that hard-negative sampling alleviates issues with CL{} \cite{jiang2024hard, robinson2020contrastive}, and (c) characterizing non-asymptotic thresholds for the minority-collapse phenomenon for more than one major class. 

\bibliographystyle{arxiv}
\bibliography{references}

@article{jiang2024hard,
  title={Hard-Negative Sampling for Contrastive Learning: Optimal Representation Geometry and Neural-vs Dimensional-Collapse},
  author={Jiang, Ruijie and Nguyen, Thuan and Aeron, Shuchin and Ishwar, Prakash},
  journal={Transactions on Machine Learning Research},
  year={2024}
}

@article{fang2021exploring,
  title={Exploring deep neural networks via layer-peeled model: Minority collapse in imbalanced training},
  author={Fang, Cong and He, Hangfeng and Long, Qi and Su, Weijie J},
  journal={Proceedings of the National Academy of Sciences},
  volume={118},
  number={43},
  pages={e2103091118},
  year={2021},
  publisher={National Academy of Sciences}
}

@inproceedings{graf2021dissecting,
  title={Dissecting supervised contrastive learning},
  author={Graf, Florian and Hofer, Christoph and Niethammer, Marc and Kwitt, Roland},
  booktitle={International Conference on Machine Learning},
  pages={3821--3830},
  year={2021},
  organization={PMLR}
}

@article{kothapalli2023neural,
title={Neural Collapse: A Review on Modelling Principles and Generalization},
author={Vignesh Kothapalli},
journal={Transactions on Machine Learning Research},
issn={2835-8856},
year={2023},
url={https://openreview.net/forum?id=QTXocpAP9p},
note={}
}

@InProceedings{pmlr-v202-dang23b,
  title = 	 {Neural Collapse in Deep Linear Networks: From Balanced to Imbalanced Data},
  author =       {Dang, Hien and Huu, Tho Tran and Osher, Stanley and Tran, Hung The and Ho, Nhat and Nguyen, Tan Minh},
  booktitle = 	 {Proceedings of the 40th International Conference on Machine Learning},
  pages = 	 {6873--6947},
  year = 	 {2023},
  editor = 	 {Krause, Andreas and Brunskill, Emma and Cho, Kyunghyun and Engelhardt, Barbara and Sabato, Sivan and Scarlett, Jonathan},
  volume = 	 {202},
  series = 	 {Proceedings of Machine Learning Research},
  month = 	 {23--29 Jul},
  publisher =    {PMLR},
  url = 	 {https://proceedings.mlr.press/v202/dang23b.html}
}

@InProceedings{pmlr-v235-dang24a,
  title = 	 {Neural Collapse for Cross-entropy Class-Imbalanced Learning with Unconstrained {R}e{LU} Features Model},
  author =       {Dang, Hien and Huu, Tho Tran and Nguyen, Tan Minh and Ho, Nhat},
  booktitle = 	 {Proceedings of the 41st International Conference on Machine Learning},
  pages = 	 {10017--10040},
  year = 	 {2024},
  editor = 	 {Salakhutdinov, Ruslan and Kolter, Zico and Heller, Katherine and Weller, Adrian and Oliver, Nuria and Scarlett, Jonathan and Berkenkamp, Felix},
  volume = 	 {235},
  series = 	 {Proceedings of Machine Learning Research},
  month = 	 {21--27 Jul},
  publisher =    {PMLR},
  url = 	 {https://proceedings.mlr.press/v235/dang24a.html}
}

@inproceedings{kini2024symmetric,
title={Symmetric Neural-Collapse Representations with Supervised Contrastive Loss: The Impact of Re{LU} and Batching},
author={Ganesh Ramachandra Kini and Vala Vakilian and Tina Behnia and Jaidev Gill and Christos Thrampoulidis},
booktitle={The Twelfth International Conference on Learning Representations},
year={2024},
url={https://openreview.net/forum?id=AyXIDfvYg8}
}

@article{hong2024neural,
  title={Neural collapse for unconstrained feature model under cross-entropy loss with imbalanced data},
  author={Hong, Wanli and Ling, Shuyang},
  journal={Journal of Machine Learning Research},
  volume={25},
  number={192},
  pages={1--48},
  year={2024}
}

@inproceedings{behnia2024supervised,
  title={Supervised Contrastive Representation Learning: Landscape Analysis with Unconstrained Features},
  author={Behnia, Tina and Thrampoulidis, Christos},
  booktitle={2024 IEEE International Symposium on Information Theory (ISIT)},
  pages={575--580},
  year={2024},
  organization={IEEE}
}

@book{royden1988real,
  title={Real Analysis 3rd Ed.},
  author={H.~L.~Royden},
  year={1988},
  publisher={Macmillan Publishing Company},
  address={New York, NY}
}

@Book{Bertsekas2002,
    author      = {Dimitri P. Bertsekas},
    title       = {Nonlinear Programming},
    address     = {Belmont, MA},
    edition     = {2nd},
    publisher   = {Athena Scientific},
    year        = {2002}
}

@book{bertsekas2010convex,
  title={Convex Optimization Theory},
  author={Bertsekas, D.P.},
  isbn={9788173717147},
  url={https://books.google.com/books?id=c80_nQAACAAJ},
  year={2010},
  publisher={Universities Press}
}

@inproceedings{robinson2020contrastive,
  title={Contrastive Learning with Hard Negative Samples},
  author={Robinson, Joshua David and Chuang, Ching-Yao and Sra, Suvrit and Jegelka, Stefanie},
  booktitle={International Conference on Learning Representations},
  year={2020}
}

@article{khosla2020supervised,
  title={Supervised contrastive learning},
  author={Khosla, Prannay and Teterwak, Piotr and Wang, Chen and Sarna, Aaron and Tian, Yonglong and Isola, Phillip and Maschinot, Aaron and Liu, Ce and Krishnan, Dilip},
  journal={Advances in Neural Information Processing Systems},
  volume={33},
  pages={18661--18673},
  year={2020}
}

@article{jaiswal2020survey,
  title={A survey on contrastive self-supervised learning},
  author={Jaiswal, Ashish and Babu, Ashwin Ramesh and Zadeh, Mohammad Zaki and Banerjee, Debapriya and Makedon, Fillia},
  journal={Technologies},
  volume={9},
  number={1},
  pages={2},
  year={2020},
  publisher={MDPI}
}

@inproceedings{dang2024neural,
  title={Neural Collapse for Cross-entropy Class-Imbalanced Learning with Unconstrained ReLU Features Model},
  author={Dang, Hien and Huu, Tho Tran and Nguyen, Tan Minh and Ho, Nhat},
  booktitle={International Conference on Machine Learning},
  pages={10017--10040},
  year={2024},
  organization={PMLR}
}

@inproceedings{wang2020understanding,
  title={Understanding contrastive representation learning through alignment and uniformity on the hypersphere},
  author={Wang, Tongzhou and Isola, Phillip},
  booktitle={International Conference on Machine Learning},
  pages={9929--9939},
  year={2020},
  organization={PMLR}
}

@misc{grant2014cvx,
  title={CVX: Matlab software for disciplined convex programming, version 2.1},
  author={Grant, Michael and Boyd, Stephen},
  year={2014}
}

@book{boyd2004convex,
  title={Convex optimization},
  author={Boyd, Stephen P and Vandenberghe, Lieven},
  year={2004},
  publisher={Cambridge university press}
}

@article{wang2023towards,
  title={Towards understanding neural collapse in supervised contrastive learning with the information bottleneck method},
  author={Wang, Siwei and Palmer, Stephanie E},
  journal={arXiv preprint arXiv:2305.11957},
  year={2023}
}

@inproceedings{jiang2024supervised,
  title={Supervised contrastive learning with hard negative samples},
  author={Jiang, Ruijie and Nguyen, Thuan and Ishwar, Prakash and Aeron, Shuchin},
  booktitle={2024 International Joint Conference on Neural Networks (IJCNN)},
  pages={1--8},
  year={2024},
  organization={IEEE}
}

@inproceedings{chen2020simple,
  title={A simple framework for contrastive learning of visual representations},
  author={Chen, Ting and Kornblith, Simon and Norouzi, Mohammad and Hinton, Geoffrey},
  booktitle={International conference on machine learning},
  pages={1597--1607},
  year={2020},
  organization={PMLR}
}

@article{nielsen2018monte,
  title={Monte Carlo information geometry: The dually flat case},
  author={Nielsen, Frank and Hadjeres, Ga{\"e}tan},
  journal={arXiv preprint arXiv:1803.07225},
  year={2018}
}

\appendix
\section{Proofs and additional supporting results}\label{apd:appendix_proofs}

\subsection{Strict convexity of the InfoNCE loss function}\label{apd: 1}
\begin{lemma}
\label{lemma: 0}
For all $i = 0, 1,\dots,k$, let $\alpha_i > 0$. Then the generalized log-sum-exponential (GLSE) function
\begin{equation}
\psi_{\text{GLSE}}(t_{1:k}) := \log \left( \alpha_0 + \sum_{i=1}^k \alpha_i e^{t_i} \right)  
\end{equation} 
is strictly convex.   
\end{lemma}
\begin{proof}
The function $\psi_{\text{GLSE}}(t_1,\dots,t_k)$ is similar to the well-known ``standard'' log-sum-exponential function \cite{boyd2004convex}.
The standard log-sum-exponential function is known to be convex, but not strictly convex. Even though the result in Lemma~\ref{lemma: 0} seems to be well-known, we are only able to find one reference that briefly mentions this result without a detailed proof \cite{nielsen2018monte}. Therefore, to make the paper self-contained, we provide the proof of Lemma \ref{lemma: 0} below.    

For all $i \in \{1:k\}$, let $u_i, v_i \in \RR$, and $w_i := (1-\lambda)u_i + \lambda v_i$, where $\lambda \in (0,1)$. Let $u_0 = v_0 = w_0 = 0$ and for some $i \in \{1:k\}$, let $u_i \neq v_i$.  If $p:= \tfrac{1}{(1-\lambda)}$ and $q := \tfrac{1}{\lambda}$, then $p, q \in (1,\infty)$, $\tfrac{1}{p} + \tfrac{1}{q} = 1$, and we have
\begin{align}
    \psi_{\text{GLSE}}(w_{1:k}) 
    &= \log\left(\alpha_0 + \sum_{i=1}^k \alpha_i e^{(1-\lambda) u_i 
       + \lambda v_i}\right) \nonumber \\
    &=  \log\left(\sum_{i=0}^k (\alpha_i e^{u_i})^{1/p} (\alpha_i e^{v_i})^{1/q}\right) \nonumber \\
    &\stackrel{\text{H\"{o}lder}}{\leq} \log\left( \Big(\sum_{i=0}^k \alpha_i e^{u_i}\Big)^{1/p}
    \Big(\sum_{i=0}^k \alpha_i e^{v_i}\Big)^{1/q}\right) \label{eq:Holder} \\
    &= (1-\lambda)\log\left(\alpha_0 + \sum_{i=1}^k \alpha_i e^{u_i}\right) + \lambda \log\left(\alpha_0 + \sum_{i=1}^k \alpha_i e^{v_i}\right) \nonumber \\
    &= (1-\lambda)\psi_{\text{GLSE}}(u_{1:k}) + \lambda \psi_{\text{GLSE}}(v_{1:k}). \nonumber
\end{align}
This shows that $\psi_{\text{GLSE}}(\cdot)$ is a convex function. Equality holds in H\"{o}lder's inequality if, and only if, for all $i \in \{0:k\}$, we have $((\alpha_i e^{u_i})^{1/p})^p = c\,((\alpha_i e^{v_i})^{1/q})^q$ for some constant $c$, i.e., $e^{u_i} = c\,e^{v_i}$, since $\alpha_i > 0$ for all $i\in \{0:k\}$ and $1/p, 1/q \in (0,1)$.
Since $u_0 = v_0 = 0$, equality can occur if, and only if, $c= 1$. This would imply that $u_i = v_i$ for all $i \in \{1:k\}$ which would contradict the assumption that for some $i \in \{1:k\}$, $u_i \neq v_i$. This proves that the inequality in (\ref{eq:Holder}) is strict and therefore $\psi_{\text{GLSE}}(\cdot)$ is a \textit{strictly} convex function.
\end{proof}

The InfoNCE loss function 
\begin{align}
\psi(t_{1:k}) =  \psi_{\text{InfoNCE}}(t_{1:k}) &:= \log\bigg(1 + \frac{1}{k} \sum_{i=1}^k e^{t_i}\bigg) \label{eq:InfoNCElossfunction}
\end{align}
is argument-wise strictly increasing and is not only convex (being a log-sum-exponential with a positive offset within the logarithm), but also strictly convex since it is a GLSE function with $\alpha_0 = 1$ and $\alpha_1 = \ldots = \alpha_k =\tfrac{1}{k} > 0$.  

\subsection{Lemmas for proving variance collapse} \label{apd:varcollapse}

\begin{lemma}\label{lemma:varcollapse0}
 Let $u, v$ be iid random vectors in $\RR^d$ with probability distribution $p(\cdot)$. 
 If $u^\top v \stackrel{\text{w.p.1}}{=} 0$, 
 then, $u \stackrel{\text{w.p.1}}{=} v \stackrel{\text{w.p.1}}{=} 0$.
\end{lemma}
\begin{proof}
 Let $\Dcal := \{1, \ldots, d+1\}$ and $w_1, \ldots, w_{d+1} \sim \text{ iid } p(\cdot)$. Since any $d+1$ vectors in $d$-dimensional space are linearly dependent, 
 \[
 \text{w.p.1. } \exists i \in \Dcal: w_i \in \text{Span}(w_{1:d+1}\setminus\{w_i\}).
 \]
 But for all $i\in\Dcal$ and all $j \in \Dcal\setminus\{i\}$, we have $w_i^\top w_j \stackrel{\text{w.p.1}}{=} 0$ since $w_i, w_j \sim \text{ iid }  p(\cdot)$. This implies that
 \[
 \exists i \in \Dcal: w_i \stackrel{\text{w.p.1}}{=} 0.
 \]
 But $u, v, w_1, \ldots, w_{d+1}$ all have the same distribution $p(\cdot)$. Therefore, $u \stackrel{\text{w.p.1}}{=}  v \stackrel{\text{w.p.1}}{=} 0$.
\end{proof}

\begin{remark} The result of Lemma~\ref{lemma:varcollapse0} is false if $u, v$ are independent, but not identically distributed, e.g., if $u = (u_1, 0)^\top, v = (0,v_2), \in \RR^2, u_1, v_2 \text{ iid standard normal}$. Clearly $u \stackrel{\text{w.p.1}}{\neq} 0$ and $v \stackrel{\text{w.p.1}}{\neq} 0$, but $u^\top v \stackrel{\text{w.p.1}}{=} 0$. Here, $u, v$ are independent, but they are not identically distributed because the first component of $u \stackrel{\text{w.p.1}}{\neq} 0$ but the second is and it is reversed for $v$.
\end{remark}

\begin{lemma}\label{lemma:varcollapse}
 Let $z_1, z_2$ be iid random vectors in $\RR^d$ and $\mu := \EE[z_1] = \EE[z_2]$. 
 If $z^\top_1 z_2 \stackrel{\text{w.p.1}}{=} \gamma$, a constant, 
 then, $z_1 \stackrel{\text{w.p.1}}{=} z_2 \stackrel{\text{w.p.1}}{=} \mu$ and $\gamma = ||\mu||^2$.
\end{lemma}
\begin{proof}
\[
z_1^\top z_2 \stackrel{\text{w.p.1}}{=} \gamma  
\Rightarrow \EE[z_1^\top z_2|z_1] \stackrel{\text{w.p.1}}{=} \gamma
\Rightarrow z_1^\top \EE[z_2|z_1] \stackrel{\text{w.p.1}}{=} \gamma
\Rightarrow z_1^\top \EE[z_2] \stackrel{\text{w.p.1}}{=} \gamma
\Rightarrow z_1^\top \mu \stackrel{\text{w.p.1}}{=} \gamma
\] 
where the last but one implication is because $z_1$ and $z_2$ are independent. Since $z_1$ and $z_2$ are also identically distributed, we have
\[
z_1^\top \mu \stackrel{\text{w.p.1}}{=} z_2^\top \mu \stackrel{\text{w.p.1}}{=} \gamma
\]
Therefore, $\EE[z_1^\top \mu] = \gamma \Rightarrow \mu^\top \mu = ||\mu||^2 = \gamma$. Next, define $u:= z_1 -\mu$ and $v := z_2 - \mu$. Then $u, v$ are iid random vectors in $\RR^d$  and $u^\top v = z_1^\top z_2 - z^\top_1 \mu - \mu^\top z_2 + \mu^\top \mu \stackrel{\text{w.p.1}}{=} \gamma - \gamma - \gamma + \gamma = 0$. By Lemma~\ref{lemma:varcollapse0}, $z_1-\mu \stackrel{\text{w.p.1}}{=} z_2 - \mu \stackrel{\text{w.p.1}}{=} 0$ which implies that $z_1 \stackrel{\text{w.p.1}}{=} z_2 \stackrel{\text{w.p.1}}{=} \mu$.
\end{proof}
\color{black}

\subsection{Proof of Lemma~\ref{lemma:CLriskLB}}\label{apd:CLriskLB}
The proof makes use of the results in Appendix~\ref{apd:varcollapse} pertaining to variance collapse. 
\begin{proof}
{\small
\begin{align}
    &L(f) = \nonumber \\ 
    &= \EE\bigg[\ell\Big(f(x),f(x^+),f(x^-_1),\ldots,f(x^-_k)\Big)\bigg] \nonumber \\
    &= \EE\bigg[\psi\Big(f^\top(x)(f(x^-_1) - f(x^+)),\ldots,f^\top(x)(f(x^-_k) - f(x^+))\Big)\bigg] \label{eq:lossdef} \\ 
    &= \EE\bigg[\EE\bigg[\psi\Big(f^\top(x)(f(x^-_1) - f(x^+)),\ldots,f^\top(x)(f(x^-_k) - f(x^+))\Big)\bigg\vert y, y^-_1,\ldots, y^-_k \bigg]\bigg] \label{eq:toalexpectation} \\ 
    &\geq \EE\bigg[\psi\Big(\EE[f^\top(x)\,f(x^-_1)|y,y^-_1] - \EE[f^\top(x)\,f(x^+)|y],\ldots,\EE[f^\top(x)\,f(x^-_k)|y,y^-_k] - \EE[f^\top(x)\,f(x^+)|y]\bigg] \label{eq:Jensen} \\ 
    &= \EE\bigg[\psi\Big(\mu_y^\top\,\mu_{y^-_1} - \EE[f^\top(x)\,f(x^+)|y],\ldots,\mu_y^\top\,\mu_{y^-_k} - \EE[f^\top(x)\,f(x^+)|y]\Big)\bigg]  \label{eq:anchnegcondindepgivenlabels} \\ 
    &= \sum_{i, j_{1:k} \in \Ccal}
{ \left( \lambda_i  \prod_{t=1}^k \lambda_{j_t} \right) } \psi  \big( \mu_i^\top \mu_{j_1} - \EE[f^\top(x)\,f(x^+)|y=i], \dots, \mu_i^\top \mu_{j_k} - \EE[f^\top(x)\,f(x^+)|y=i]\big) \label{eq:JensentogetG} \\
    &\geq \sum_{i,j_{1:k} \in \Ccal}
{ \left( \lambda_i  \prod_{t=1}^k \lambda_{j_t} \right) }  \psi  \big( \mu_i^\top \mu_{j_1} - 1, \dots, \mu_i^\top \mu_{j_k} - 1 \big), \label{eq:labeljoint}
\end{align}
}
where equality (\ref{eq:lossdef}) follows from (\ref{eq:genCLloss}), 
equality (\ref{eq:toalexpectation}) is the law of total expectation, 
inequality (\ref{eq:Jensen}) is Jensen's inequality applied  within the inner expectation conditioned on the labels of samples to the convex loss function $\psi(\cdot)$, 
(\ref{eq:anchnegcondindepgivenlabels}) follows from the conditional independence of anchor and negative samples given their labels implied by (\ref{eq:joint2of2}), 
(\ref{eq:JensentogetG}) follows by expanding the expectation in (\ref{eq:anchnegcondindepgivenlabels}) in terms of all possible tuples of values of labels together with (\ref{eq:joint1of2}), 
and inequality (\ref{eq:labeljoint})  is because $\psi(\cdot)$ is an increasing function of all its arguments, all the weights $\big(\lambda_i \prod_{t=1}^k \lambda_{j_t}\big)$ are positive, and $f^\top(x)\,f(x^+) \leq ||f(x)||\cdot ||f(x^+)|| \leq 1$ since the representations are constrained to be within the unit ball. 
\paragraph{}Clearly, if $f$ is such that $\forall x \in \Xcal, f(x) = \mu_{y(x)}$ and $\forall i \in \Ccal, ||\mu_i||^2 = 1$, then $L(f) = G(M)$. 
We will now prove that these conditions are also necessary for equality. If $L(f) = G(M)$, then we must have equality in (\ref{eq:Jensen}) and (\ref{eq:labeljoint}). Equality in (\ref{eq:labeljoint}) can be attained only if $\forall i \in \Ccal$, with probability one (w.p.1) given $y=i$, i.e., under the distribution $q(x,x^+|i)$, we have $f^\top(x)\,f(x^+) = 1$. This is because $\psi(\cdot)$ is a \textit{strictly} increasing function of its arguments,  all the weights $\big(\lambda_i \prod_{t=1}^k \lambda_{j_t}\big)$ are \textit{strictly} positive, and the norms of all representations are bounded by one. Therefore, w.p.1 given $y=i$, we must have $f(x) = f(x^+)$ and $||f(x)|| = 1$. 
Next, equality in the conditional Jensen's inequality (\ref{eq:Jensen}) can be attained only if $\forall i \in \{1:k\}$, w.p.1 given $y,y^-_i$, we have $f^\top(x)f(x^-_i) - f^\top(x)f(x^+) = \mu_y^\top\,\mu_{y^-_i} - \EE[f^\top(x)\,f(x^+)|y]$. This is because $\psi(\cdot)$ is a \textit{strictly} convex function of its arguments and for all label tuples, $p(y,y^-_{1:k}) > 0$. This implies that $\forall i \in \{1:k\}$ and all $j, l \in \Ccal$, w.p.1 given $y = j,y^-_i = l$, we have $f^\top(x)f(x^-_i)  = \mu_j^\top\,\mu_l$ since, as we previously proved, equality in (\ref{eq:labeljoint}) implies that w.p.1 given $y=j$, we must have $f(x) = f(x^+)$ and $||f(x)|| = 1$. Taking $j = l$, we conclude that equality in (\ref{eq:labeljoint}) and (\ref{eq:Jensen}) imply that  $\forall i \in \{1:k\}$ and all $j \in \Ccal$, w.p.1 given $y = y^-_i = j$, we have $f^\top(x)f(x^-_i)  = ||\mu_j||^2$. But $x, x^-_i$ are conditionally iid with distribution $s(\cdot|j)$ given $y = y^-_i = j$. From Lemma~\ref{lemma:varcollapse} in Appendix~\ref{apd:varcollapse}, it then follows that for all $j\in \Ccal$, w.p.1 given $y=j$, $f(x) = \mu_j$, or more compactly, $\forall x\in\Xcal, f(x) = \mu_{y(x)}$. Thus we have shown that the conditions $\forall x \in \Xcal, f(x) = \mu_{y(x)}$ and $\forall i \in \Ccal, ||\mu_i||^2 = 1$ are both sufficient and necessary for the lower bound $G(M)$ to be attained, i.e., for $L(f) = G(M)$.
\phantom{\hfill}
\end{proof}
In the proof of necessity of within-class variance collapse for the attainment of the lower bound in Lemma~\ref{lemma:CLriskLB}, as an intermediate step we first proved that if we have equality in (\ref{eq:labeljoint}), then for each $i \in \Ccal$, w.p.1 given $y=i$, we must have $f(x) = f(x^+)$ and $||f(x)|| = 1$. Without making any additional assumptions on the joint distribution of the positive pair, specifically, $q(x,x^+|y)$, we cannot conclude from here that we must have within-class variance collapse. For example, if $x^+=x$ w.p.1, or if the samples in each class are grouped into non-overlapping pairs and $x,x^+$ are confined to be within a pair. But if, for example, the support of $q(x,x^+|y)$ is the Cartesian product of the supports of $s(x|y)$ and $s(x^+|y)$, then indeed we can conclude within-class variance collapse directly from equality in (\ref{eq:labeljoint}) alone without needing to analyze the conditions for equality in (\ref{eq:Jensen}).

\subsection{Proof of Corollary~\ref{cor:InfoNCElb_infinite_k}}\label{apd:proof_InfoNCElb_infinite_k}
\begin{proof}
For $t = 1:k$, let 
\begin{align*}
    U_t(x,x^+,x^-_t) &:=  e^{f^\top(x)(f(x^-_t)-f(x^+))}, \\
    V_t(y,y^-_t) &:= e^{(\mu^\top_y \mu_{y^-_t})  - 1 }.
\end{align*}
Then, from (\ref{eq:lossdef}), the definition of the InfoNCE loss function in (\ref{eq:InfoNCElossfunction}), and (\ref{eq:Jensen}), (\ref{eq:labeljoint}), and (\ref{eq: thuan101}) we have
\begin{align}
    L(f) &= \EE\left[\EE\bigg[\log\bigg(1+\frac{1}{k}\sum_{t=1}^k U_t(x,x^+,x^-_t)\bigg)\bigg| x,x^+ \bigg]\right], \label{eq:LintermsofU} \\
    G(M) &= \EE\left[\EE\bigg[\log\bigg(1+\frac{1}{k}\sum_{t=1}^k V_t(y,y^-_t)\bigg)\bigg| y \bigg]\right]. \label{eq:GintermsofV}
\end{align}
Since for all $x\in\Xcal$, $\|f(x)\|\leq 1$, it follows from the convexity of the Euclidean norm and Jensen's inequality that for all $j\in \Ccal$, $\|\mu_j\| = \|\EE[f(x)|y=j]\| \leq \EE[\|f(x)\|\, |y=j] \leq 1$ and therefore (by the Cauchy-Schwartz inequality) $|f^\top(x)f(x^-_t)|, |f^\top(x)f(x^+)| \leq 1$. This proves that for all $t = 1:k$, $|U_t|, |V_t| \leq e^2$, i.e., they are bounded random variables. Now, $U_{1:k}|x,x^+$ and $V_{1:k}|y$ are conditionally iid. Thus, by the Strong Law of Large Numbers, their averages converge w.p.1 to their respective conditional expectations, i.e., 
\begin{align}
     \frac{1}{k} \sum_{t=1}^k U_t \underset{k \rightarrow \infty}{\stackrel{\text{w.p.1}}{\longrightarrow}} \EE[U_1|x,x^+] 
    &= \EE\Big[e^{f^\top(x)(f(x^-_1)-f(x^+))}\Big|x,x^+\Big], \label{eq:SLLNforU} \\
    \frac{1}{k} \sum_{t=1}^k V_t \underset{k \rightarrow \infty}{\stackrel{\text{w.p.1}}{\longrightarrow}} \EE[V_1|x,x^+] 
    &= \EE\Big[e^{(\mu^\top_y \mu_{y^-_1})- 1}\Big|y\Big]
    = \sum_{j \in \Ccal_y }  r(j|y)\, e^{(\mu_y^\top \mu_j) - 1}. \label{eq:SLLNforV} 
\end{align}
Since $U_{1:k}$ and $V_{1:k}$ are bounded by $e^2$ so are $(\sum_{t=1}^kU_t)/k$ and  $(\sum_{t=1}^kV_t)/k$. The results (\ref{eq:infinitekL}) and (\ref{eq:infinitekG}) then follow from (\ref{eq:LintermsofU}), (\ref{eq:GintermsofV}), (\ref{eq:SLLNforU}), (\ref{eq:SLLNforV}), the Dominated Convergence Theorem, and the fact that $L(f) \geq G(M)$ proved in Lemma~\ref{lemma:CLriskLB}.
\end{proof}

\subsection{Proof of Theorem~\ref{theorem:non-neg}}\label{apd:proof_of_theorem_non-neg}
\begin{proof}
From Lemma~\ref{lemma:CLriskLB}, $L(f) \geq G(M)$ with equality if, and only if, $\forall x \in \Xcal$, $f(x) = \mu_{y(x)}$ and $\mu_{1:C} \in \Mcal$. For any $u, v \in \RR^d_{\geq 0}$, $u^\top v \geq 0$ with equality only if $u$ and $v$ are orthogonal. In (\ref{eq: thuan101}), $\psi(\cdot)$ is a \textit{strictly} increasing function of its arguments and all the weights $\big(\lambda_i \prod_{t=1}^k \lambda_{j_t}\big)$ are \textit{strictly} positive and sum to one. Therefore, $G(M)$ is minimized over $M \in \Mcal \cap \RR^{d\times C}_{\geq 0}$ if, and only if, $\mu_{1:C}$ are orthonormal. This requires $d \geq C$.  
\end{proof}

\subsection{Proof of Lemma~\ref{lemma: equivalent}} \label{apd:proof_of_lemma-equivalent}
\begin{proof}
If $A := M^\top M$, then clearly, $A = A^\top$ and $A \succcurlyeq 0$ (since for all $u \in \RR^C, u^\top M^\top M u = ||M u||^2 \geq 0$), and $\forall i \in \Ccal, A_{ii} = ||\mu_i||^2 = 1$. Therefore $A = M^\top M \in \Acal^*$. From (\ref{eq: thuan30}) and (\ref{eq:Sdefn}), which define $G(\cdot)$ and $S(\cdot)$ respectively, it follows that $G(M) = S(A) = S(M^\top M)$. 
Therefore, for any $M\in \Mcal$ we have $G(M) = S(M^\top M) \geq S(A^*) = S((M^*)^\top M^*) = G(M^*)$. This shows that $M^*$ is a solution to (\ref{eq: thuan30}). 
\end{proof}

\subsection{Proof of Lemma~\ref{lemma:convexproblem}} \label{apd:proof_lemma_convexproblem}
\begin{proof} 
\textit{Convexity and compactness of $\Acal^*$:} The set $\Acal^*$ is clearly convex, since the set of all symmetric PSD matrices in $\RR^{C\times C}$ satisfying the specified unit diagonal equality constraints is convex. The set $\Acal^*$ is also compact since $\Acal^* \subset \RR^{C\times C}$ and for any $A\in\Acal^*$ and all $i,j\in\Ccal$, $|A_{ij}| \leq |A_{ii}|\cdot|A_{jj}| = 1$, as we prove next.
Since $A$ is real, symmetric, and PSD, by the Real Spectral Theorem it has an eigendecomposition given by $A = U\Sigma U^\top$. If $\sqrt{A} := U\sqrt{\Sigma}U^\top$, where $\sqrt{\Sigma} \in \RR^{C\times C}$ is a diagonal matrix with the square roots of $C$ non-negative eigenvalues of $A$ along the main diagonal, then $\sqrt{A} \cdot \sqrt{A} = A$. If $e_{1:C}$ is the standard basis for $\RR^C$, then  
$|A_{ij}| = |e_i^\top A e_j| = |e_i^\top \sqrt{A} \sqrt{A} e_j| \leq ||\sqrt{A} e_i|| \cdot ||\sqrt{A} e_j|| = \sqrt{e_i^\top \sqrt{A} \sqrt{A} e_i} \cdot \sqrt{e_j^\top \sqrt{A} \sqrt{A} e_j} = \sqrt{e_i^\top A e_i} \cdot \sqrt{e_j^\top A e_j} = A_{ii} \cdot A_{jj} = 1$, where the first inequality is the Cauchy-Schwartz inequality. Thus, for all $i, j \in \Ccal$, we have $|A_{i,j}| \leq 1$. This shows that $\Acal^*$ is a compact set. \newline \\
\textit{Strict convexity of $S(\cdot)$ over $\Acal^*$:} Let $A \in \Acal^*$. In (\ref{eq:Sdefn}), for all $i, j_{1:k}\in \Ccal$, the $k$-tuples $\big( A_{ij_1}-1, \dots, A_{ij_k}-1 \big)$ are linear functions of $A$ and the weights $\big(\lambda_i \prod_{t=1}^k \lambda_{j_t}\big)$  are all non-negative (in fact, they are all strictly positive). Since the function $\psi(\cdot)$ is convex (in fact, it is strictly convex), and $S(A)$ is a positive linear combination of convex functions of linear functions of $A$, it follows that $S(A)$ is a convex function of $A$. To prove that $S(\cdot)$ is strictly convex over $\Acal^*$,  let $A, B \in \Acal^*, A\neq B$. Since $\forall i \in \Ccal$, $A_{ii} = B_{ii} = 1$, we must have $A_{ij} \neq B_{ij}$ for at least one $i\neq j, i, j \in \Ccal$. For any $t \in (0,1)$, let $W := (1-t)A + t B$. Then, $W \in \Acal^*$ since $\Acal^*$ is a convex set and $A, B \in \Acal^*$, and $\forall i \in \Ccal, W_{ii} = (1-t)A_{ii} + tB_{ii} = 1$. Since $\psi(\cdot)$ is a convex function of its arguments, for all tuples $(i,j_{1:k}) \in \Ccal^{k+1}$, we will have
\begin{align*}
(1-t)\,\psi \big( A_{ij_1}-&A_{ii}, \dots, A_{ij_k}-A_{ii} \big) + t\,\psi \big(B_{ij_1}-B_{ii}, \dots, B_{ij_k}-B_{ii} \big) \\
&= (1-t)\,\psi \big( A_{ij_1}-1, \dots, A_{ij_k}-1 \big) + t\,\psi \big( B_{ij_1}-1, \dots, B_{ij_k}-1 \big) \\
&\geq \psi \big( (1-t)\,(A_{ij_1} - 1) + t\,(B_{ij_1} -1), \dots, (1-t)\,(A_{ij_k} - 1) + t\,(B_{ij_k} -1)  \big)  \\
&= \psi \big( W_{ij_1} -1, \dots, W_{ij_k} - 1 \big)  \\
&= \psi \big( W_{ij_1} - W_{ii}, \dots, W_{ij_k} - W_{ii} \big)
\end{align*}
and the inequality is strict for at least one tuple $(i,j_{1:k}) \in \Ccal^{k+1}$ because $\psi(\cdot)$ is a \textit{strictly} convex function of its arguments, $A \neq B$, and $t \notin \{0,1\}$. Since the weights $\big(\lambda_i \prod_{t=1}^k \lambda_{j_t}\big)$ in (\ref{eq:Sdefn}) are all strictly positive, it follows that $S(\cdot)$ is a strictly convex function over $\Acal^*$. 
\end{proof}

\subsection{Proof of Lemma~\ref{lemma:rank_of_opt_Gram_mtx}}
\label{apd:proof_rank_of_opt_Gram_mtx}
\begin{proof}
Let $\underline{\nu}(\cdot)$ denote the minimum eigenvalue of a matrix. We will prove that $\underline{\nu}(A^*) = 0$. For all $t > 0$, let 
\[
B(t) := A^* - t\,\bfones \bfones^\top + t\, I
\]
where $\bfones$ is the $C\times 1$ vector of all ones and $I$ is the $C\times C$ identity matrix. For all $t$, $B(t)$ is symmetric since $A^*, \bfones \bfones^\top$, and $I$ are symmetric matrices. For all $i\in \Ccal$, $B_{ii}(t) = A^*_{ii} - t + t = 1$ and for all $i, j \in \Ccal, i \neq j$, $B_{ij}(t) = A^*_{ij} - t + 0 < A^*_{ij}$. Since $\psi(\cdot)$ is a \textit{strictly} increasing function of all its arguments and all the weights $\lambda_i \prod_{t=1}^k (\lambda_{j_t}/(1-\lambda_i))$ in (\ref{eq:Sdefn}) are strictly positive, it follows that $S(B) < S(A^*)$. We now show that if $\underline{\nu}(A^*) > 0$, then $B(t)$ is PSD for $t = t' := \tfrac{\underline{\nu}(A^*)}{2(C-1)}$. This would imply that $B(t') \in \Acal^*$ and contradict the optimality of $A^*$.  By the Courant-Fischer min-max theorem,
\begin{align}
    \underline{\nu}(B(t)) &= \min_{u\neq 0} \frac{u^\top B(t) u}{||u||^2} 
    = \min_{u\neq 0} \frac{u^\top (A^* - t\,\bfones \bfones^\top + t\, I) u}{||u||^2} 
    = \min_{u\neq 0} \frac{u^\top A^* u  - t\,(u^\top\bfones)^2  + t||u||^2}{||u||^2} \nonumber \\    
    &\geq \min_{u\neq 0} \frac{u^\top A^* u  - t\,C ||u||^2  + t||u||^2}{||u||^2} \label{eq:Cauchy-Schwartz} \\ 
    &= \min_{u\neq 0} \frac{u^\top A^* u}{||u||^2}  - (C-1)t  
    =  \underline{\nu}(A^*) - (C-1)t, \nonumber
\end{align}
where (\ref{eq:Cauchy-Schwartz}) is due to the Cauchy-Schwartz inequality. Therefore, $\underline{\nu}(B(t')) \geq \tfrac{\underline{\nu}(A^*)}{2}$. Thus, if $\underline{\nu}(A^*) > 0$, then $\underline{\nu}(B(t')) > 0$ which would make $B(t')$ a PSD matrix and contradict the optimality of $A^*$. We must therefore conclude that $\underline{\nu}(A^*) = 0$ which implies that $\text{rank}(A^*) < C$. 
\end{proof}

\subsection{Proof of Theorem~\ref{theorem: 1}}\label{apd:proof_of_theorem-opt_class_means}
\begin{proof}
Lemma~\ref{lemma:rank_of_opt_Gram_mtx} proved that (\ref{eq: thuan30}) has a unique solution $A^*$ in $\Acal^*$ with rank $r$ less than or equal to $C-1$. Since $A^*$ is also a real, symmetric, PSD matrix, by the Real Spectral Theorem, it has a reduced eigen-decomposition given by $A^* = U_r \Sigma_r U_r^\top$. For all $d \geq C-1$, the matrix  $(M^*)^\top := [ U_r \sqrt{\Sigma_r}\ \ \ 0_{\text{\tiny{\(C \times d-r+1\)}}}]$ is well defined and 
$(M^*)^\top M^* = U_r\sqrt{\Sigma} (\sqrt{\Sigma})^\top U_r^\top  + 0_{\text{\tiny{\(C \times d-r+1\)}}} 0_{\text{\tiny{\(C \times d-r+1\)}}}^\top =  U_r \Sigma_r U_r^\top = A^*$. From Lemma~\ref{lemma: equivalent} it follows that $M^*$ is a solution to (\ref{eq: thuan30}). Moreover, for all $i\in \Ccal$, we have $||\mu^*_i||^2 = A^*_{ii} = 1$.
\end{proof}

\subsection{Proof of Theorem~\ref{theorem:equiangular-properties}}\label{apd:proof_of_theorem_equiangular-properties}
\begin{proof}
The key idea of the proof is to show that if we swap $\mu_i^*$ and $\mu_j^*$ in $M^*$ to form a new matrix $Q$, then $S(Q^\top Q) = S({M^*}^\top M^*)$. By construction, the gram matrix $B := Q^\top Q \in \Acal^*$ since $A^* = {M^*}^\top M^* \in \Acal^*$. Since the optimal Gram matrix is unique,  $B =  Q^\top Q =  {M^*}^\top M^* = A^*$ and therefore for all $n \in \Ccal \setminus \{i,j\}$, we must have $B_{jn} = {\mu_i^*}^\top \mu_n^* = A^*_{jn} = {\mu_{j}^*}^\top \mu_{n}^*$. \newline
It remains to show that $S(Q^\top Q) = S( {M^*}^\top M^*)$, i.e., $S(A^*) = S(B)$. To this end, let $\sigma: \Ccal \rightarrow \Ccal$ denote the bijection (specifically, a transposition permutation) where $\sigma(i) = j, \sigma(j) = i$, and for all $n \in \Ccal\setminus\{i,j\}, \sigma(n) = n$. Then, $\sigma(\cdot)$ is its own inverse, i.e., $\forall n \in \Ccal,\,\sigma(\sigma(n)) = n$. For notational convenience, let primed-indices denote the image under $\sigma(\cdot)$, i.e., $n' := \sigma(n)$. 
By construction of $Q$ and the definition of $\sigma(\cdot)$, we have
\begin{align}
\forall j_1, j_2 \in \Ccal,\, A^*_{j'_1j'_2} 
= B_{j_1j_2}. \label{eq:BintermsofA}
\end{align}
Since $\lambda_i = \lambda_j$, it follows from the definition of $\sigma(\cdot)$ that
\begin{align}
    \forall n \in \Ccal,\, \lambda_{n'} = \lambda_{\sigma(n')} = \lambda_{n}. \label{eq:lambdaprimes}
\end{align}
Therefore,
\begin{align}
    S(A^*) &=  \sum_{j_0, j_{1:k} \in \Ccal}
{ \left( \lambda_{j_0}  \prod_{t=1}^k \lambda_{j_t}  \right) }  \psi  \big( A^*_{j_0j_1}-A^*_{j_0j_0}, \dots, A^*_{j_0j_k}-A^*_{j_0j_0} \big). \label{eq:duetoSdefn} \\
&= \sum_{j'_0, j'_{1:k} \in \Ccal}
{ \left( \lambda_{j'_0}  \prod_{t=1}^k \lambda_{j'_t} \right) }  \psi  \big( A^*_{j'_0j'_1}-A^*_{j'_0j'_0}, \dots, A^*_{j'_0j'_k}-A^*_{j'_0j'_0} \big). \label{eq:duetoequivalentindexsets} \\
&= \sum_{j'_0, j'_{1:k} \in \Ccal}
{ \left( \lambda_{j_0}  \prod_{t=1}^k \lambda_{j_t} \right) }  \psi  \big( B_{j_0j_1}-B_{j_0j_0}, \dots, B_{j_0j_k}-B_{j_0j_0} \big). \label{eq:duetoBintermsofA} \\
&= \sum_{j_0, j_{1:k} \in \Ccal}
{ \left( \lambda_{j_0}  \prod_{t=1}^k \lambda_{j_t} \right) }  \psi  \big( B_{j_0j_1}-B_{j_0j_0}, \dots, B_{j_0j_k}-B_{j_0j_0} \big). \label{eq:duetoequivalentindexsets2}  \\
&= S(B), \label{eq:SofB}
\end{align}
where (\ref{eq:duetoSdefn}) follows from the definition of $S(\cdot)$ in (\ref{eq:Sdefn}), equality (\ref{eq:duetoequivalentindexsets}) holds because $\sigma(\cdot)$ is a bijection, (\ref{eq:duetoBintermsofA}) is due to (\ref{eq:BintermsofA}) and (\ref{eq:lambdaprimes}),  equality (\ref{eq:duetoequivalentindexsets2}) holds because $\sigma(\cdot)$ is a bijection, and (\ref{eq:SofB}) again follows from the definition of $S(\cdot)$ in (\ref{eq:Sdefn}).
\end{proof}

\subsection{Proof of Corollary~\ref{cor: 1}}\label{apd:proof_of_corollary-1}
\begin{proof}
The Corollary follows directly by applying the result in Theorem~\ref{theorem:equiangular-properties} to different pairs of $(i,j) \in \Ccal'$ as follows. If $\Ccal'$ contains only two classes, then the proof is immediate. If $\Ccal'$ contains more than two classes, consider any three distinct classes $i,j,n \in \Ccal'$. Then, from Theorem~\ref{theorem:equiangular-properties} we have (1) ${\mu_n^*}^\top \mu_j^* = {\mu_n^*}^\top \mu_i^*$ since $\lambda_i = \lambda_j$ and (2) ${\mu_i^*}^\top \mu_j^* = {\mu_n^*}^\top \mu_j^*$ since $\lambda_i = \lambda_n$.  Therefore, ${\mu_i^*}^\top \mu_j^*={\mu_n^*}^\top \mu_j^*={\mu_n^*}^\top \mu_i^*$. In other words, any pair of class means has the same inner product.  
\end{proof}

\subsection{Proof of Corollary~\ref{cor:ETF}}\label{apd:proof_ETF}
\begin{proof}
If $\Ccal' = \Ccal$ in Corollary~\ref{cor: 1}, then for all $i, j \in \Ccal, i\neq j$, we have ${\mu_i^*}^\top \mu_j^* = b$ for some constant $b$. This implies that $A^* \in \Acal^*$ has the following form 
\begin{align}
A^*&= (1-b) I + b \bfones\,\bfones^\top = 
\begin{bmatrix}
1 & b & b & \cdots & b \\
b & 1 & b & \cdots & b \\
\vdots & \vdots & \vdots & \ddots & \vdots \\
b & b & b & \cdots & 1
\end{bmatrix} 
\in \RR^{C \times C}, 
\label{eq:ETFAstarform}
\end{align}
where $I$ is the $C \times C$ identity matrix and $\bfones \in \RR^C$ is the all-ones column vector. A matrix $A^*$ having the above form has $(C-1)$ eigenvalues equal to $(1-b)$ and one eigenvalue equal to $(C-1)b+1$. Since $A^*$ is PSD, $b \in [-1/(C-1), 1]$. By Lemma~\ref{lemma:rank_of_opt_Gram_mtx}, the smallest eigenvalue of $A^*$ is zero which implies that either $1-b = 0 \Rightarrow b = 1$ or $(C-1)b+1 = 0 \Rightarrow b = -1/(C-1)$. For both choices of $b$, $A^*$ is PSD, but for the choice $b = -1/(C-1)$ (the smaller choice), the value of $S(A^*)$ is smaller because for $A^*$ having the form in (\ref{eq:ETFAstarform}), 
\[
S(A^*) = \sum_{i, j_{1:k} \in \Ccal}
{ \left( \lambda_i  \prod_{t=1}^k \lambda_{j_t} \right) }  \psi  \big( (b-1)1(j_1 \neq i), \dots, (b-1)1(j_k \neq i) \big),
\]
where $1(\cdot)$ is the indicator function, and $\psi$ is a strictly increasing function of all it arguments. Thus $b = -1/(C-1)$. Finally, $||\sum_{i\in\Ccal}\mu^*_i||^2 = \sum_{i\in \Ccal}||\mu^*_i||^2 + \sum_{i\neq j, i, j \in \Ccal} (\mu^*_i)^\top \mu^*_j = C -C(C-1)/(C-1) = 0$.
\end{proof}

\subsection{Structure of Optimum $A^*$}\label{apd:Astarform}
\begin{lemma}\label{lemma:Astarform}
Let $C\geq 3$ and $1 > \lambda_1 > \lambda_2=\lambda_3=\ldots=\lambda_C = \tfrac{1-\lambda_1}{C-1}>0$. Then 
\begin{align}
A^*&=
\begin{bmatrix}
1 & a & a & \cdots & a \\
a & 1 & b & \cdots & b \\
a & b & 1 & \cdots & b \\
\vdots & \vdots & \vdots & \ddots & \vdots \\
a & b & b & \cdots & 1
\end{bmatrix} 
\in \RR^{C \times C}, \label{eq:Astarform}
\end{align}
with $a\in [-1,1]$ and $b = (a^2(C-1)-1)/(C-2)$.
\end{lemma}
The form of $A^*$ in (\ref{eq:Astarform}) follows from Theorem~\ref{theorem:equiangular-properties} and Corollary~\ref{cor: 1}. The condition on $b$ follows from the rank deficiency of $A^*$ proved in Lemma~\ref{lemma:rank_of_opt_Gram_mtx}. This requires a careful analysis of the eigenstructure of PSD matrices having the form in (\ref{eq:Astarform}). The detailed proof is presented below. 

\begin{proof}
From Theorem~\ref{theorem:equiangular-properties} and Corollary~\ref{cor: 1}, it follows that 
$\forall i \in \Ccal\setminus\{1\},\, {\mu^*_i}^\top \mu^*_1 = a$ and $\forall i,j \in \Ccal\setminus\{1\},\,i\neq j,\, {\mu^*_i}^\top \mu^*_j = b$ for some constants $a, b \in [-1,1]$. Thus, $A^* \in \RR^{C \times C}$ is of the form
\begin{align}
A^*&=
\begin{bmatrix}
1 & a & a & \cdots & a \\
a & 1 & b & \cdots & b \\
a & b & 1 & \cdots & b \\
\vdots & \vdots & \vdots & \ddots & \vdots \\
a & b & b & \cdots & 1
\end{bmatrix}. \label{eq:appAstarform}
\end{align}
Since $A^* \in \Acal^*$, it is PSD and all its eigenvalues are non-negative. From Lemma~\ref{lemma:rank_of_opt_Gram_mtx}, the minimum eigenvalue of $A^*$ is zero. We will show that this implies either $a =0$ and $b=1$ or $a\in [-1,1]$ and $b = (a^2(C-1)-1)/(C-2)$.

To this end, let $\bfones \in \RR^C$ denote the all-ones column vector and $e_1 \in \RR^C$ the standard basis vector whose first component is one and the remaining components are zero. Let $u:= \bfones - e_1$. Then, $u \perp e_1$ and 
\begin{align}
    A^* = (1-b)\,I + b\,u\,u^\top + b\,e_1\,e_1^\top + a\,e_1\,u^\top + a\,u\,e_1^\top, \label{eq:Astaradditivedecomposition}
\end{align}
where $I$ is the $C\times C$ identity matrix. Let $v_{1:C}$ be any orthonormal basis for $\RR^C$ with $v_1 := e_1$, $v_2 := u/||u||$, and $v_{3:C} \in \text{Span}^\perp(e_1,u)$. Then using (\ref{eq:Astaradditivedecomposition}), it follows that for all $i \geq 3$, 
\[
A^*v_i = (1-b) v_i = 0.
\]
This shows that $v_{3:C}$ are $(C-2)$ orthonormal eigenvectors of $A^*$ with eigenvalue $(1-b)$. The remaining two eigenvectors of $A^*$ must therefore belong to  $\text{Span}(e_1,u)$. Let $v = \alpha e_1 + \beta u$ be an eigenvector of $A^*$ in $\text{Span}(e_1,u)$ with eigenvalue $\nu \geq 0$. Then $v = (\alpha\ \beta \ldots \beta)^\top$ and either $\alpha \neq 0$ or $\beta \neq 0$ because, by definition, an eigenvector is a non-zero vector. Since $A^*\,v = \nu\,v$ and $A^*$ has the form shown in (\ref{eq:appAstarform}), we have
\begin{align}
    \alpha + (C-1)\,a\,\beta = \nu\,\alpha &\Rightarrow  (C-1)\,a\,\beta = -(1-\nu)\,\alpha \label{eq:alphabetaeq1} \\
    a\,\alpha + \beta + (C-2)\,b\,\beta = \nu\,\beta &\Rightarrow \beta((1-\nu)+(C-2)b) = -a\,\alpha \label{eq:alphabetaeq2}
\end{align}

\noindent\textbf{Case $a=0$.} Then, $b\neq 0$ since otherwise we would have $A^* = I$ which has $C$ eigenvalues all equal to one and this would contradict the result of Lemma~\ref{lemma:rank_of_opt_Gram_mtx}. With $a=0, b\neq 0$, (\ref{eq:alphabetaeq1}) would imply that $(1-\nu)\,\alpha = 0$ which would imply that either $\nu = 1$ or $\alpha=0$. If $\nu = 1$, then (\ref{eq:alphabetaeq2}) together with $a = 0$ and $b\neq 0$ would imply that $\beta=0$ which would, in turn, imply that $\alpha \neq 0$ since both $\alpha$ and $\beta$ cannot be simultaneously zero. Thus, when $a=0$, one eigenvalue is $\nu=1$ with eigenvector given by $\alpha \neq 0, \beta = 0$. If $a = 0$ and we have $\nu \neq 1$, then $\alpha = 0$, $\beta \neq 0$, and $(1-\nu) + (C-2)b=0 \Rightarrow \nu = 1 + (C-2)b$. In summary, if $a=0$ then $b\neq 0$ and $A^*$ would have $(C-2)$ eigenvalues equal to $(1-b)$, one eigenvalue equal to $1$, and one eigenvalue equal to $1 + (C-2)b$. Since the smallest eigenvalue of $A^*$ is zero, this would imply that either $b=1$ or $b = -1/(C-2)$. 

\noindent \textbf{Case $a\neq 0$.} In this case we must have $\nu \neq 1$ because otherwise (\ref{eq:alphabetaeq1}) and $C\geq 3$ would imply that $\beta = 0$ and then (\ref{eq:alphabetaeq2}) would imply that $\alpha = 0$ which would contradict the assumption that both $\alpha$ and $\beta$ cannot be zero simultaneously. Thus, $\nu \neq 1$. Then, (\ref{eq:alphabetaeq1}) would imply that $\alpha = -(C-1)a\,\beta/(1-\nu)$. Substituting this into (\ref{eq:alphabetaeq2}) gives us 
\[
 \beta((1-\nu)+(C-2)b) = \beta\,\frac{(C-1)\,a^2}{(1-\nu)} \Rightarrow (1-\nu)^2 + (C-2)\,b\,(1-\nu) - (C-1)\,a^2 = 0, 
\]
where we could cancel the common factor $\beta$ in the first equation because $\beta \neq 0$ (if $\beta = 0$ then with $\nu \neq 1$, (\ref{eq:alphabetaeq1}) would imply that $\alpha = 0$, a contradiction). Solving for the roots of the quadratic equation in $(1-\nu)$ we get
\begin{align}
    \nu = 1 + \frac{(C-2)\,b}{2} \pm \sqrt{\frac{(C-2)^2b^2}{4} + (C-1)a^2} \label{eq:quadraticequation}
\end{align}
In summary, if $a\neq 0$, then $A^*$ would have $(C-2)$ eigenvalues equal to $(1-b)$ and two  eigenvalues given by (\ref{eq:quadraticequation}). Since the smallest eigenvalue of $A^*$ is zero, this would imply that either $b=1$ or 
\begin{align}
1 + \frac{(C-2)\,b}{2} - \sqrt{\frac{(C-2)^2b^2}{4} + (C-1)a^2} = 0 \Rightarrow b = \frac{(C-1)\,a^2 - 1}{(C-2)}. \label{eq:bintermsofa}    
\end{align}

Observe that if we substitute $a=0$ into the expression for $b$ in terms of $a$ given by (\ref{eq:bintermsofa}), we get $b = -1/(C-2)$, which is consistent with one of the two possibilities that we obtained when we previously analyzed the case $a=0$. Combining the analysis of both cases, we conclude that we must have either $a = 0, b=1$ or $a\in [-1,1],\,b = ((C-1)\,a^2-1)/(C-2)$. 

Since $\psi(\cdot)$ is a \textit{strictly} increasing function of all its arguments and all the weights $\lambda_i \prod_{t=1}^k (\lambda_{j_t}/(1-\lambda_i))$ in (\ref{eq:Sdefn}) are strictly positive, $S(A^*)$ will have a strictly smaller value when $a = 0, b = -1/(C-1)$ than when $a = 0, b = 1$. Therefore, we must have $a\in [-1,1]$, $b = ((C-1)\,a^2-1)/(C-2)$.
\end{proof}

\color[rgb]{0.0,0.5,1.0}
\color{black}
%
\subsection{Subgradients of strictly convex and argument-wise strictly increasing functions}\label{apd:subgradients}
\begin{lemma}\label{lemma:subgradients}
Let $\psi:\RR^k \longrightarrow \RR$ be strictly convex and argument-wise strictly increasing. Then for all $v \in \RR^k$, the subdifferential set $\partial \psi (v)$ is non-empty, convex, and compact. 
Moreover, if $\Vcal := [-2,0]^k$, then $\Scal(\psi) := \cup_{v\in \Vcal} \partial \psi(v)$ is bounded and $\psi$ is Lipschitz over $\Vcal$. Specifically, if   
\begin{align}
\Delta_2 := \underset{w\,\in\,\Scal(\psi)}{\sup}\,||w||_2, \quad \text{then } 
\Delta_2 < \infty, \quad 
\Scal(\psi) \subseteq (0, \Delta_2]^k, \text{ and } \nonumber \\
\forall\, v, v' \in \Vcal,\quad |\psi(v)-\psi(v')| \leq \Delta_2\, ||v-v'||_2. \label{eq:psiLipschitz} 
\end{align}
For all $u \in \RR^k_{\geq 0}\setminus\{\bfzero\}$ and all $t \in [-2,0]$, let $\phi_u(t) := \psi(t\,u)$. Then, 
\begin{align}
\forall\, u \in \RR^k_{\geq 0}\setminus\{\bfzero\},\,\exists\,\delta_u \in (0, \infty)\,:\, \forall\, t, t' \in [-2,0],\, t'\leq t,\quad  (t - t')\,\delta_u \leq \phi_u(t) - \phi_u(t'). \label{eq:phislopelowerbound}    
\end{align}
If $u = \bfzero$, then for all $t$, $\phi_{\bfzero}(t) = \psi(\bfzero)$ and we define $\delta_{\bfzero} := 0$. If $\nabla \psi(v)$ exits for all $v\in \Vcal$, then $\Delta_2 = \sup_{v\in \Vcal} ||\nabla \psi(v)||_2$ and $\forall\, u \in \RR^k_{\geq 0}$,  $\delta_u =  u^\top \nabla\psi (-2u)$.
\end{lemma}
The proof essentially follows from standard results in convex optimization theory, the fact that $\psi$ is argument-wise strictly increasing, and the definition of subgradients and subdifferentials. The detailed proof is presented below. 
%
\begin{proof} Proposition~5.4.2 in \citep{bertsekas2010convex} and Proposition~B.24 in Appendix~B of \citep{Bertsekas2002} prove that the subdifferential set $\partial \psi(v)$ at any point $v\in \RR^k$ of any real-valued convex function $\psi: \RR^k \longrightarrow \RR$, is non-empty, convex, and compact. Moreover, the union of subdifferential sets of all points belonging to any non-empty compact set $\Vcal$ is also bounded, i.e., $\cup_{v \in \Vcal} \partial \psi(v)$ is bounded. 

In the lemma, we have $\Vcal = [-2,0]^k$ which is a non-empty compact set.
Therefore, $\Scal(\psi) := \cup_{v \in \Vcal} \partial \psi(v)$ is  bounded and $\Delta_2 := \sup_{w \in \Scal(\psi)}||w||_2 < \infty$. For any vector $w \in \RR^k$ we have $||w||_{\infty} \leq ||w||_2$. This implies that for all $w \in \Scal(\psi)$, we have $||w||_{\infty} \leq \Delta_2$.
Since $\psi$ is also strictly increasing over $\RR^k$, all components of any subgradient vector at any point are strictly positive. Specifically, for all $v \in \Vcal$, all subgradients $w\in \partial \psi(v)$,  all $i\in \Ccal$, and all $t > 0$, we have (by the definition of a subgradient)
\[
-t(e_i^\top w) + \psi(v)  \leq  \psi(v-t\,e_i)
\]
where $e_i$ is the $i^{\text{th}}$ standard basis vector of $\RR^k$. Thus, the  $i^{\text{th}}$ component of $w$ is bounded from below as follows
\[
(e_i^\top w) \geq \frac{\psi(v) - \psi(v-t\,e_i)}{t} > 0
\]
where the last inequality is strict since $\psi$ is argument-wise strictly increasing and $t > 0$.   
Therefore, we conclude that $\Scal(\psi) \subseteq (0,\Delta_2]^k$. 

Next, for all $v, v \in [-2,0]^k$, all $w \in \partial\psi(v)$, and all $w' \in \partial\psi(v')$, by the definition of a subgradient, the fact that $||w||_2, ||w'||_2 \leq \Delta_2 < \infty$, and the Cauchy-Schwartz inequality, we have
\begin{align*}
-\Delta_2 ||v-v'||_2 \leq -||w'||_{2} \cdot ||v'-v||_2  &\leq (v-v')^\top w' \leq \psi(v) - \psi(v') \\
&\leq (v'-v)^\top w \leq ||w||_{2} \cdot ||v'-v||_2 \leq \Delta_2 ||v'-v||_2.
\end{align*}
Thus, $|\psi(v)-\psi(v')| \leq \Delta_2\, ||v-v'||_2$. If $\nabla\psi(v)$ exists for all $v \in \Vcal$, then $\Scal(\psi) = \{\nabla\psi(v):v\in\Vcal\}$ and $\Delta_2 = \sup_{v\in \Vcal}||\nabla\psi(v)||_{2}$. 

Since $\psi$ is strictly convex and argument-wise strictly increasing over $\RR^k$, it follows that $\forall u \in \RR^k_{\geq 0}\setminus\{\bfzero\}$, $\phi_u(t) := \psi(t\,u)$ is also strictly convex and strictly increasing over $\RR$ (strictly, because at least one component of $u$ is strictly positive). According to the ``chord-slopes inequality'' for convex functions (see \citep{royden1988real}, Chapter~5, Section~5), if $\phi: \RR \longrightarrow \RR$ is convex, then for all $s_1, s_2, s'_1, s'_2 \in \RR$ such that $s_1 \leq s_1' < s_2'$ and $s_1 < s_2 \leq s'_2$, we have 
\[
\frac{\phi(s_2) - \phi(s_1)}{s_2 - s_1} \leq \frac{\phi(s'_2) - \phi(s'_1)}{s'_2 - s'_1}.
\]
Applying this inequality to $\phi_u$ with $s'_2 = t$, $s'_1 = t'$, with $-2 \leq t' < t \leq 0$, and $s_2 = -2$, and $s_1 = -2-\epsilon$, where $\epsilon > 0$, we get
\[
\frac{\phi_u(-2) - \phi_u(-2-\epsilon)}{(-2) - (-2-\epsilon)} \leq \frac{\phi_u(t) - \phi_u(t')}{t - t'}.
\]
Since $\phi_u$ is a strictly increasing function, we get
\[
0 < \delta_u := \sup_{\epsilon > 0}\left[\frac{\phi_u(-2) - \phi_u(-2-\epsilon)}{\epsilon}\right] \leq \frac{\phi_u(t) - \phi_u(t')}{t - t'}.
\]
Thus, for all $t, t' \in [-2,0]$, with $t' < t$, we have
\[
(t-t')\,\delta_u \leq \phi_u(t) - \phi_u(t').
\]
The last inequality clearly holds when $t' = t$ as well.

If $\nabla\psi(v)$ exists for all $v \in \Vcal$, then 
\[
\delta_u = u^\top\nabla\psi(-2\,u), 
\]
since for all $\epsilon  > 0$, the convexity of $\phi_u$ implies that  $-\epsilon\,\nabla \phi_u(-2) + \phi_u(-2) \leq \phi_u(-2-\epsilon) \Rightarrow \tfrac{\phi_u(-2) - \phi_u(-2-\epsilon)}{\epsilon} \leq \nabla\phi_u(-2) = u^\top\nabla\psi(-2\,u)$, and $\lim_{\epsilon \downarrow 0}\tfrac{\phi_u(-2) - \phi_u(-2-\epsilon)}{\epsilon} = \nabla\phi_u(-2)$.
\end{proof}

\subsection{$A^*(a)$ is Lipschitz}\label{apd:proof_Astarbound}
\begin{lemma}\label{lemma:Astarbound}
Let $C \geq 3$ and let $A^*(a)$ denote the matrix $A^*$ in Equation~(\ref{eq:Astarform}) of Lemma~\ref{lemma:Astarform} with $a\in [-1,1]$ and $b = (a^2(C-1)-1)/(C-2)$. Then, for all $a,a' \in [-1,1]$ such that $a'\leq a$, and all $i,j \in \Ccal$, we have
\[
|A^*_{ij}(a) - A^*_{ij}(a')| \leq \gamma_C \cdot (a-a'),
\]
where $\gamma_C := \tfrac{2(C-1)}{(C-2)}$, and for all $i \in \Ccal$ and all $j_{1:k} \in \Ccal\setminus\{i\}$,
\begin{align*}
|\psi(A^*_{ij_1}(a)-A^*_{ii}(a),\ldots,A^*_{ij_k}(a)-A^*_{ii}(a)) - \psi(A^*_{ij_1}(a')-A^*_{ii}(a'),\ldots,&A^*_{ij_k}(a')-A^*_{ii}(a'))| \\
&\leq \gamma_C\,\Delta_2 \sqrt{k}(a-a'), 
\end{align*}
where $\psi$ and $\Delta_2$ are as in Lemma~\ref{lemma:subgradients}.
\end{lemma}
%
%
%
\begin{proof} 
For all $i=j \in \Ccal$, $A^*_{ii}(a) = 1$, a constant, irrespective of the value of $a \in [-1,1]$. Therefore, for all $a, a' \in [-1,1]$ such that $a' \leq a$, we have $|A^*_{ii}(a) - A^*_{ii}(a')| = 0 \leq \tfrac{2(C-1)}{(C-2)}(a-a') = \gamma_C\cdot (a-a')$. Note that $1 < \tfrac{\gamma_C}{2} = \tfrac{(C-1)}{(C-2)} < \infty$ since $C \geq 3$.
Now consider any $i, j \in \Ccal$ with $i \neq j$. If
either $i = 1$ or $j=1$, then for all $a \in [-1,1]$, $A^*_{ij}(a) = a$ and therefore $|A^*_{ij}(a) - A^*_{ij}(a')|  = |a - a'| = (a-a') \leq \gamma_C\cdot (a-a')$. If $i \neq 1$ and $j \neq 1$ and $i\neq j$, then for all  $a \in [-1,1]$, $A^*_{ij}(a) = b = (a^2(C-1)-1)/(C-2)$ and then,
\[
|A^*_{ij}(a) - A^*_{ij}(a')| = \frac{|a^2 - (a')^2|(C-1)}{(C-2)} 
= \frac{(a + a')(a-a')(C-1)}{(C-2)} \leq \tfrac{2(C-1)}{(C-2)} (a-a') = \gamma_C\cdot (a-a').
\]
This proves that for all $a' \leq a$ with $a, a' \in [-1,1]$, and all $i, j \in \Ccal$, we have $|A^*_{ij}(a) - A^*_{ij}(a')| \leq \gamma_C \cdot (a-a')$.
Next, for all $i\in \Ccal$, all $j_{1:k} \in \Ccal\setminus\{i\}$, and all $a \in [-1,1]$, let
\[
v(a) := (A^*_{ij_1}(a)-A^*_{ii}(a),\ldots,A^*_{ij_k}(a)-A^*_{ii}(a))^\top = (A^*_{ij_1}(a) - 1,\ldots,A^*_{ij_k}(a) - 1)^\top.
\]
Then, for all $a, a'\in [-1,1]$ with $a' \leq a$, the bound on $|A^*_{ij}(a)-A^*_{ij}(a')|$ that we just proved implies that
\[
||v(a) - v(a')||_2 = \sqrt{\sum_{m=1}^k |A^*_{ij_m}(a) - A^*_{ij_m}(a')|^2} \leq \sqrt{\sum_{m=1}^k (\gamma_C\cdot\,(a-a'))^2} = \gamma_C\,\sqrt{k}\,\,(a-a'). 
\]
Therefore, from  Lemma~\ref{lemma:subgradients}, we get
\[
|\psi(v(a)) - \psi(v(a'))|\leq   \Delta_2 ||v(a)-v(a')||_2 \leq \gamma_C\,\sqrt{k}\,\Delta_2 (a-a'). 
\]
\end{proof}

\subsection{Proof of Theorem~\ref{theorem:minoritycollapse}}\label{apd:proof_minoritycollapse}
The proof makes use of the results in Lemma~\ref{lemma:Astarform}, Lemma~\ref{lemma:subgradients}, and Lemma~\ref{lemma:Astarbound} which appear in Appendix~\ref{apd:Astarform}, Appendix~\ref{apd:subgradients}, and Appendix~\ref{apd:proof_Astarbound}, respectively. 
\begin{proof} 
Let $\Ecal_{1\bar{1}} := \{y=1 \mbox{ and for some } i,\, y^-_i \neq 1\}$, $\Ecal_{\bar{1}1} := \{y\neq 1 \mbox{ and for all } i,\, y^-_i = 1\}$, and $\Ecal_1 := \Ecal_{1\bar{1}} \cupdot \Ecal_{\bar{1}1}$.
Let $\Ecal_{=} := \{y=y^-_1=\ldots=y^-_k\}$ and $\Ecal_{2} := (\Ecal_{=}\cup\Ecal_1)^c$. Then, $\Ecal_{=}$, $\Ecal_{1}$, and $\Ecal_{2}$ are mutually exclusive and exhaustive events with
\begin{align}
    \Pr(\Ecal_{=}) &= \lambda_1^{k+1} + (C-1)\,\frac{(1-\lambda_1)^{k+1}}{(C-1)^{k+1}} = \lambda_1^{k+1} + \frac{(1-\lambda_1)^{k+1}}{(C-1)^{k}} 
    \nonumber \\
    \Pr(\Ecal_1) &= \Pr(\Ecal_{1\bar{1}}) + \Pr(\Ecal_{\bar{1}1}) = \lambda_1(1-\lambda_1)^k + (1-\lambda_1)\lambda_1^k = \lambda_1\,(1-\lambda_1)\,\bigg(\frac{(1-\lambda_1)^k}{1-\lambda_1} + \lambda_1^{k-1}\bigg),  \label{eq:probE1} \\
    \Pr(\Ecal_2) &= 1 - \Pr(\Ecal_{=}) - \Pr(\Ecal1) = (1-\lambda_1)^2 \bigg(\frac{(1-\lambda_1)^k}{1-\lambda_1} - \frac{(1-\lambda_1)^{k-1}}{(C-1)^{k}}\bigg), 
    \nonumber \\
    \frac{\Pr(\Ecal_1)}{\Pr(\Ecal_2)} &\geq \frac{\lambda_1}{1-\lambda_1} \label{eq:probE12ratiolbnd}.
\end{align}
Next, noting the definition of $\phi$ in Lemma~\ref{lemma:subgradients} and that for all $j \in \Ccal$,  $(A^*_{1j} - 1) = (A^*_{j1} - 1) = (a-1)\,1(j\neq 1)$, we have 
$\forall\, (y,y^-_{1:k}) \in \Ecal_{=}\cupdot\Ecal_1$,
\begin{align}
\psi(A^*_{yy^-_1}(a) - 1, \ldots,A^*_{yy^-_1}(a) - 1) = \psi((a - 1)\, u(y,y^-_{1:k})) = \phi_{u(y,y^-_{1:k})}(a-1), \label{eq:psiintermsofphi}
\end{align}
and in particular for all $(y,y^-_{1:k}) \in \Ecal_{=}$, $u(y,y^-_{1:k}) =  \bfzero$ and $\phi_{u(y,y^-_{1:k})}(a-1) = \psi(0,\ldots,0) = \phi_{\bfzero}(0)$, a constant. We also note that for all $(y,y^-_{1:k}) \in \Ecal_1$, $u(y,y^-_{1:k}) \neq \bfzero$.
For all $a, a' \in [-1,1]$ with $a' < a$ and $y, y^-_{1:k}$ distributed as in (\ref{eq:joint1of2}), we have
\begin{flalign}
S(A^*(a)) &= \EE\Big[\psi(A^*_{yy^-_1}(a) - 1, \ldots,A^*_{yy^-_1}(a) - 1)\Big] \nonumber && \\
 &= \Pr(\Ecal_{=})\,\phi_{\bfzero}(0) + \Pr(\Ecal_{1})\,\EE[\phi_{u(y,y^-_{1:k})}(a-1)|\Ecal_1] \quad + \nonumber && \\ 
& \hspace{31ex} \Pr(\Ecal_2)\,\EE\Big[\psi(A^*_{yy^-_1}(a) - 1, \ldots,A^*_{yy^-_1}(a) - 1)\Big|\Ecal_2\Big]. \nonumber &&
\end{flalign}
Therefore,
\begin{align}
&S(A^*(a)) - S(A^*(a')) \nonumber \\
&= \Pr(\Ecal_1)\,\EE\Big[\phi_{u(y,y^-_{1:k})}(a-1) - \phi_{u(y,y^-_{1:k})}(a'-1)|\Ecal_1\Big] \quad + \nonumber \\ 
&\hspace{3ex}\Pr(\Ecal_2)\,\EE\Big[\psi(A^*_{yy^-_1}(a) - 1, \ldots,A^*_{yy^-_1}(a) - 1) - \psi(A^*_{yy^-_1}(a') - 1, \ldots,A^*_{yy^-_1}(a') - 1) \Big|\Ecal_2\Big] \nonumber \\
&\geq \Pr(\Ecal_1)\,\EE\big[\delta_{u(y,y^-_{1:k})}\,(a-a')|\Ecal_1\big] - \Pr(\Ecal_2)\,\gamma_C\,\sqrt{k}\,\Delta_2\,(a-a') \label{eq:Slowerbound1} \\
&= (a-a')\,\Pr(\Ecal_2)\,\left\{\frac{\Pr(\Ecal_1)}{\Pr(\Ecal_2)}\,\EE\big[\delta_{u(y,y^-_{1:k})}|\Ecal_1\big]  - \gamma_C \sqrt{k} \Delta_2\right\} \nonumber \\
&\geq (a-a')\,\frac{\Pr(\Ecal_2)}{1-\lambda_1}\,\left\{\lambda_1\,\EE\big[\delta_{u(y,y^-_{1:k})}|\Ecal_1\big]  - (1-\lambda_1)\,\gamma_C \sqrt{k} \Delta_2\right\} \label{eq:Slowerbound2}  \\
&= (a-a')\,\frac{ \gamma_C \sqrt{k} \Delta_2\,\Pr(\Ecal_2)}{1-\lambda_1}\,\left\{\lambda_1\,\bigg(1+\frac{1}{\gamma_C \sqrt{k} \Delta_2}\EE\big[\delta_{u(y,y^-_{1:k})}|\Ecal_1\big]\bigg) - 1\right\} \nonumber \\
&\geq 0. \label{eq:Slowerbound3} 
\end{align}
Inequality (\ref{eq:Slowerbound1}) follows from (\ref{eq:phislopelowerbound}) and Lemma~\ref{lemma:Astarbound} together with the fact that $s\geq - |s|$ for all $s \in \RR$. Inequality (\ref{eq:Slowerbound2}) follows from (\ref{eq:probE12ratiolbnd}).
Inequality (\ref{eq:Slowerbound3}) follows from condition (\ref{eq:MClambda1suffcond}) and the assumption that $a'<a$. 

Thus, if condition (\ref{eq:MClambda1suffcond}) is satisfied, then for all $a\in [-1,1]$, $S(A^*(a))$ is a strictly increasing function of the variable $a$ and is minimized when $a = -1$. When $a = -1$, $b = (a^2(C-1)-1)/(C-2) = 1$. Then, $\forall i \in \Ccal\setminus\{1\}$, $(\mu^*_i)^\top \mu^*_1 = a = -1$. Since for all $j\in \Ccal$ we have $||\mu^*_j||^2 = 1$, it follows from the alignment conditions for equality in the Cauchy-Schwartz inequality that for all $i \in \Ccal\setminus\{1\}$, $\mu^*_i = -\mu^*_1$. 
Finally, if condition $(\ref{eq:lambda1threshold})$ is satisfied, then condition (\ref{eq:MClambda1suffcond}) is also satisfied because 
\[
\lambda_1 \geq \tau \Rightarrow \lambda_1 \geq \tfrac{1}{1 + \frac{\delta_*}{\gamma_C\,\sqrt{k}\,\Delta_2}} \geq \tfrac{1}{1+ \frac{1}{\gamma_C\,\sqrt{k}\,\Delta_2}\EE\big[\delta_{u(y,y^-_{1:k})}|\Ecal_1\big]}
\]
and the last inequality holds because for all $(y,y^-_{1:k}) \in \Ecal_1$, $u(y,y^-_{1:k}) \neq \bfzero$, and by the definition of $\delta_*$ in (\ref{eq:mindelta}), for all $u \neq \bfzero$, $\delta_u \geq \delta_* > 0$. 
\end{proof}

\subsection{Proof of Corollary~\ref{cor:InfoNCElambda1threshold}}\label{apd:proof_InfoNCElambda1threshold}
\begin{proof}
From (\ref{eq:MClambda1suffcond}), a sufficient condition for minority collapse is given by 
\[
\lambda_1 \geq \frac{1}{1+ \frac{1}{\gamma_C \sqrt{k} \Delta_2}\,\EE[\delta_{u(y,y^-_{1:k})}| \Ecal_1]}.
\]
For the InfoNCE loss function, we will show that $\Delta_2 = 1/(2\sqrt{k})$ and develop a lower bound for $\EE\big[\delta_{u(y,y^-_{1:k})}| \Ecal_1\big]$ which is independent of $k$. This would yield a sufficient threshold for minority collapse. For the InfoNCE loss function,
\begin{align*}
\psi(t_{1:k}) = \log\Big(1+\tfrac{1}{k}\sum_{i=1}^ke^{t_i}\Big) &\Rightarrow
\nabla\psi^\top(t_{1:k}) = \tfrac{1}{k+\sum_{i=1}^k e^{t_i}} (e^{t_1},\ldots,e^{t_k}) \\
&\Rightarrow ||\nabla\psi(t_{1:k})||_2 = \sqrt{\frac{\sum_{i=1}^k (e^{t_i})^2}{\big(k+\sum_{i=1}^k e^{t_i}\big)^2}}.
\end{align*}
For all $v_{1:k} \in \RR$ we have
\[
0 \leq \bigg(k-\sum_{i=1}^k v_i\bigg)^2 
\Rightarrow 2k\bigg(\sum_{i=1}^k v_i\bigg) \leq k^2 + \bigg(\sum_{i=1}^k v_i\bigg)^2  
\Rightarrow 4k\bigg(\sum_{i=1}^k v_i\bigg) \leq \bigg(k + \sum_{i=1}^k v_i\bigg)^2.
\]
Therefore, for all  $v_{1:k} \in [0,1]$,
\[
4k\bigg(\sum_{i=1}^k v_i^2\bigg) \leq 4k\bigg(\sum_{i=1}^k v_i\bigg) \leq \bigg(k + \sum_{i=1}^k v_i\bigg)^2 
\Rightarrow \sqrt{\frac{\sum_{i=1}^k v_i^2}{\big(k+\sum_{i=1}^k v_i \big)^2}} \leq \frac{1}{2\sqrt{k}}
\]
with equality if, and only if, $\forall i, v_i = 1$. Thus, for all $t_{1:k}\in[-2,0]$,  with $v_i := e^{t_i} \in [e^{-2},1]$, we get
\[
\Delta_2 = \sup_{t_{1:k}\in[-2,0]} ||\nabla\psi(t_{1:k})||_2 
= \sup_{v_{1:k}\in[e^{-2},1]} \sqrt{\frac{\sum_{i=1}^k v_i^2}{\big(k+\sum_{i=1}^k v_i \big)^2}} = \frac{1}{2\sqrt{k}} \Rightarrow \gamma_C\,\sqrt{k}\,\Delta_2 = \tfrac{1}{2}\gamma_C.
\]
Thus, a sufficient condition for minority collapse is given by
\[
\lambda_1 \geq \frac{1}{1+\frac{2}{\gamma_C}\,\EE\big[\delta_{u(y,y^-_{1:k})}| \Ecal_1\big]}.
\]
We will now develop a lower bound for $\EE\big[\delta_{u(y,y^-_{1:k})}| \Ecal_1\big]$ which is independent of $k$. By Lemma~\ref{lemma:subgradients}, for all $u \in \{0,1\}^k$,
\begin{align*}
\delta_u &= u^\top\,\nabla\psi(-2\,u) \\
&= \sum_{i=1}^k u_i \frac{e^{-2u_i}}{k + \sum_{j=1}^k e^{-2u_j}} \\
&= \frac{e^{-2} ||u||_1}{k+e^{-2} ||u||_1 + (k-||u||_1)} \\
&= \frac{||u||_1}{2ke^2-(e^{2}-1) ||u||_1} \\
&=: g(||u||_1),
\end{align*}
and we note that $\delta_u = g(||u||_1)$ is an increasing function of $||u||_1$. From Theorem~\ref{theorem:minoritycollapse}, $\Ecal_1 := \Ecal_{1\bar{1}} \cupdot \Ecal_{\bar{1}1}$, 
for all $(y,y^-_{1:k}) \in \Ecal_{1}$, $u(y,y^-_{1:k}) \neq \bfzero$, and for all $(y,y^-_{1:k}) \in \Ecal_{1\bar{1}}$, $y=1$ and $(y^-_{1:k}) \neq (1,\ldots,1)$. Moreover, from (\ref{eq:probE1}),
\[
\Pr(\Ecal_1) = \lambda_1\,(1-\lambda_1)^k + \lambda_1^k\,(1-\lambda_1) \leq \lambda_1\,(1-\lambda_1) + \lambda_1\,(1-\lambda_1) = 2\lambda_1\,(1-\lambda_1) \leq \tfrac{1}{2}.  
\] 
Therefore,
\begin{align}
2\lambda_1\,&(1-\lambda_1)\,\EE\big[\delta_{u(y,y^-_{1:k})}|\Ecal_1\big] \nonumber \\
&\geq \Pr(\Ecal_1)\,\EE\big[\delta_{u(y,y^-_{1:k})}|\Ecal_1\big]\,, \nonumber \\
&= \sum_{(y,y^-_{1:k}) \in \Ecal_1} p(y,y^-_{1:k})\,\delta_{u(y,y^-_{1:k})}\,, \nonumber \\ 
&\geq \sum_{(y,y^-_{1:k}) \in \Ecal_{1\bar{1}}} p(y,y^-_{1:k})\,\delta_{u(y,y^-_{1:k})}\,, \nonumber \\
&= \sum_{(y,y^-_{1:k}) \in \Ecal_{1\bar{1}}} \lambda_1\,\Big(\prod_{i=1}^{k} \lambda_{y^-_i}\Big)\,\delta_{u(y,y^-_{1:k})}\,, \nonumber \\
&= \lambda_1\,\sum_{(y^-_{1:k}) \in \Ccal^k\setminus \{\bfones\}} \lambda_1^{k-||u(1,y^-_{1:k})||_1} \Big(\frac{1- \lambda_{1}}{C-1}\Big)^{||u(1,y^-_{1:k})||_1}\,g(||u(1,y^-_{1:k})||_1)\,, \nonumber \\
&= \lambda_1\,\sum_{w\in \{0,1\}^k\setminus\{\bfzero\}}\,\sum_{y^-_{1:k}:u(1,y^-_{1:k}) = w} \lambda_1^{k-||w||_1} \Big(\frac{1- \lambda_{1}}{C-1}\Big)^{||w||_1}\,g(||w||_1)\,, \nonumber \\
&= \lambda_1\,\sum_{w\in \{0,1\}^k\setminus\{\bfzero\}}\,(C-1)^{||w||_1}\,\lambda_1^{k-||w||_1} \Big(\frac{1- \lambda_{1}}{C-1}\Big)^{||w||_1}\,g(||w||_1)\,, \nonumber \\   
&= \lambda_1\,\sum_{w\in \{0,1\}^k\setminus\{\bfzero\}}\,\lambda_1^{k-||w||_1}\,(1- \lambda_{1})^{||w||_1}\,g(||w||_1)\,, \nonumber \\ 
&= \lambda_1\,\sum_{l=1}^k\,\binom{k}{l} \,\lambda_1^{k-l}\,(1- \lambda_{1})^{l}\,g(l)\,, \nonumber \\
&= \lambda_1\,\sum_{l=0}^k\,\binom{k}{l} \,\lambda_1^{k-l}\,(1- \lambda_{1})^{l}\,g(l), \quad \text{since } g(0) = 0\,, \nonumber \\
&= \lambda_1\,\EE[g(l)], \quad l \sim \text{Binomial}(k,1-\lambda_1)\,, \nonumber \\   
&\geq \lambda_1\,\EE[1(l \geq k/2)\,g(l)], \quad l \sim \text{Binomial}(k,1-\lambda_1)\,, \quad \text{since } \forall l, g(l) \geq 0, \nonumber \\ 
&\geq \lambda_1\,\EE[1(l \geq k/2)\,g(k/2)], \quad l \sim \text{Binomial}(k,1-\lambda_1)\,, \quad \text{since } g(l) \text{ increases with } l, \nonumber \\   
&= \lambda_1\,g(k/2)\,\Pr(l \geq k/2), \quad l \sim \text{Binomial}(k,1-\lambda_1)\,, \nonumber \\     
&\geq \lambda_1\,\frac{1}{2}\,g(k/2), \nonumber
\end{align}
\begin{align}
\Rightarrow \EE\big[\delta_{u(y,y^-_{1:k})}|\Ecal_1\big] &\geq  \frac{1}{4(1-\lambda_1)}\,g(k/2) \nonumber \\
&= \frac{1}{4(1+3e^2)(1-\lambda_1)}\nonumber 
\end{align}
Therefore, a sufficient condition for minority collapse is given by
\begin{align*}
\lambda_1 \geq \frac{1}{1+\frac{2}{\gamma_C}\,\frac{1}{4(1+3e^2)(1-\lambda_1)}} = \frac{1}{1+\frac{1}{2\,\gamma_C\,(1+3e^2)(1-\lambda_1)}} 
&\Rightarrow 0 \geq \lambda_1^2 - 2\frac{\lambda_1}{\beta_C} + 1 \\
&\Rightarrow \lambda_1 \geq \frac{1-\sqrt{1-\beta_C^2}}{\beta_C} =: \tau_C,
\end{align*}
where $\beta_C := \frac{1}{1 + \tfrac{1}{4\gamma_C(1+3e^2)}}$. We note that $\tau_C \in (0,1)$ since $\beta_C \in (0,1)$. \newline
Since $(C-1) = (C-2) + 1$ and $C \geq 3$, we have $\gamma_C = \tfrac{2(C-1)}{(C-2)} \in (2,4]$. The most conservative (maximum) value of $\tau_C$ occurs when $\beta_C$ is maximum (since $\tau_C$ is an increasing function of $\beta_C$) which occurs when $\gamma_C$ is maximum (since $\beta_C$ is an increasing function of $\gamma_C$), which occurs when $C$ is minimum, i.e., $C=3$. When $C=3$, $\gamma_C = 4, \beta_C \approx 0.9973$, and $\tau_C \approx 0.9292$. Thus, $\lambda_1 \in \Big(0.9292,1\Big)$ is a sufficient condition for minority collapse for the InfoNCE loss, which holds for all $C\geq 3$ and all $k$.
\end{proof}

\subsection{Extension of theoretical results to the SCL setting}\label{apd:extensions_to_SCL}

In the SCL setting, for all $y \in \Ccal$ we have $y^-_{1:k} \in \Ccal \setminus \{y\} \text{ w.p.1}$,
\begin{equation}
p_{SCL}(y,y^-_{1:k}) := \lambda_y\prod_{t=1}^k \left(\frac{\lambda_{y^-_t}}{1-\lambda_y}\right),\label{eq:SCL_joint_pmf_labels_appendix} 
\end{equation}
and for all $x, x^+ \in \Xcal, y \in \Ccal$,
\[
q(x,x^+|y) = s(x|y)\,s(x^+|y).
\]

With the above changes, all results in Section~\ref{sec: lower bound}, Section~\ref{sec: main result}, and Section~\ref{sec: equiangular} hold with all summations $\sum_{y,y^-_{1:k} \in \Ccal}$ replaced by $\sum_{y \in \Ccal, y^-_{1:k} \in \Ccal \setminus \{y\} }$ and all products $\prod_{t=1}^k \lambda_{y^-_t}$ replaced by 
$\prod_{t=1}^k \frac{\lambda_{y^-_t}}{1-\lambda_{y}}$. With these changes, the proofs of all results in Section~\ref{sec: lower bound}, Section~\ref{sec: main result}, and Section~\ref{sec: equiangular}  go through in a straightforward manner with the exception of the proof of necessity of within-class variance collapse in Lemma~\ref{lemma:CLriskLB} which requires additional elaboration.  

As in the proof of the \UCL{} setting, equality in (\ref{eq:labeljoint}) can be attained only if $\forall i \in \Ccal$, w.p.1 given $y=i$, we must have $f(x) = f(x^+)$ and $||f(x)|| = 1$ (here, all the weights $\lambda_y\prod_{t=1}^k \left(\frac{\lambda_{y^-_t}}{1-\lambda_y}\right)$ are \textit{strictly} positive). In the \SCL{} setting, $x,x^+$ are conditionally iid with distribution $s(\cdot|j)$ given $y=j$. From Lemma~\ref{lemma:varcollapse} in Appendix~\ref{apd:varcollapse}, it then follows that for all $j\in \Ccal$, w.p.1 given $y=j$, $f(x) = \mu_j$, or more compactly, $\forall x\in\Xcal, f(x) = \mu_{y(x)}$ completing the proof of necessity in the \SCL{} setting.

The proofs of all subsequent results in Section~\ref{sec: main result} and Section~\ref{sec: equiangular} go through straightforwardly since they only make use of the lower bound  $G(M)$ in Lemma~\ref{lemma:CLriskLB}.

In the \SCL{} setting, Lemma~\ref{lemma:Astarform}, Lemma~\ref{lemma:subgradients}, and Lemma~\ref{lemma:Astarbound} in Section~\ref{sec: minority collapse} and their proofs in the appendices hold without any changes. However, Theorem~\ref{theorem:minoritycollapse} and Corollary~\ref{cor:InfoNCElb} and their proofs change slightly in the \SCL{} setting as described below.

\begin{theorem}[Sufficient conditions for minority collapse in the \SCL{} setting] \label{theorem:SCLminoritycollapse}
Let $C, \lambda_{1:C}$ be as in Lemma~\ref{lemma:Astarform}, $S(\cdot)$ be as in (\ref{eq:Sdefn}), $\Delta_2$, $\phi$, and $\delta_u$ be as in Lemma~\ref{lemma:subgradients} and let $a, b, A^*(a)$, and $\gamma_C$ be as in Lemma~\ref{lemma:Astarbound}.
With $(y, y^-_{1:k})$ distributed as in (\ref{eq:SCL_joint_pmf_labels}), if
\begin{align}
\lambda_1 \geq \tau := \frac{1}{1+ \frac{1}{\gamma_C \sqrt{k} \Delta_2}\,\delta_{\bold{1}}} \in (0,1),
\label{eq:SCLMClambda1suffcond}    
\end{align}
where $\bold{1}$ is the $C\times 1$ vector of all ones, then for all $a\in [-1,1]$, $S(A^*(a))$ is a strictly increasing function of the variable $a$ and is minimized when $a = -1 \Rightarrow b = 1$ and then for all $i \in \Ccal\setminus\{1\}$, $\mu^*_i = -\mu^*_1$ with $||\mu^*_1||=1$, i.e., we have minority collapse. 
\end{theorem}

\begin{proof} For all $a, a' \in [-1,1]$ with $a' < a$ and $(y, y^-_{1:k})$ distributed as in (\ref{eq:SCL_joint_pmf_labels}) we have
\begin{align}
S(A^*(a)) 
&= \EE\Big[\psi(A^*_{yy^-_1}(a) - A^*_{yy}(a), \ldots,A^*_{yy^-_1}(a) - A^*_{yy}(a))\Big] \nonumber \\
&= \EE\Big[\psi(A^*_{yy^-_1}(a) - 1, \ldots,A^*_{yy^-_1}(a) - 1)\Big] \nonumber \\
&= \lambda_1\,\EE\Big[\psi(A^*_{1y^-_1}(a) - 1, \ldots,A^*_{1y^-_1}(a) - 1)\Big|y=1\Big] \quad + \nonumber \\
&\hspace{22ex} (1-\lambda_1)\,\EE\Big[\psi(A^*_{yy^-_1}(a) - 1, \ldots,A^*_{yy^-_1}(a) - 1)\Big|y\neq 1\Big] \nonumber \\
&= \lambda_1\,\EE\Big[\psi(a - 1, \ldots,a - 1)\Big|y=1\Big] \quad + \nonumber \\ 
&\hspace{22ex} (1-\lambda_1)\,\EE\Big[\psi(A^*_{yy^-_1}(a) - 1, \ldots,A^*_{yy^-_1}(a) - 1)\Big|y\neq 1\Big] \nonumber \\
&= \lambda_1\,\phi_{\bold{1}}(a-1) + (1-\lambda_1)\,\EE\Big[\psi(A^*_{yy^-_1}(a) - 1, \ldots,A^*_{yy^-_1}(a) - 1)\Big|y\neq 1\Big], \label{eq:psitophi}
\end{align}
where the second and third equalities are because all the diagonal entries of the matrix $A^*(a)$ in Equation~(\ref{eq:Astarform}) of Lemma~\ref{lemma:Astarform} are equal to one and if $y=1$ then for all $m \in \{1:k\}$, $y^-_m \neq 1$ which would imply that $A^*_{yy^-_m}(a) = A^*_{1y^-_m}(a) = a$. Equation~(\ref{eq:psitophi}) follows from the definition of $\phi_{u}$ in Lemma~\ref{lemma:subgradients}.
Therefore,
\begin{align}
S(A^*(a)) &- S(A^*(a')) \nonumber \\
&= \lambda_1\,(\phi_{\bold{1}}(a-1) - \phi_{\bold{1}}(a'-1))\ + (1-\lambda_1)\,\EE\Big[\psi(A^*_{yy^-_1}(a) - 1, \ldots,A^*_{yy^-_1}(a) - 1) - \nonumber \\
&\hspace{34ex}\psi(A^*_{yy^-_1}(a') - 1, \ldots,A^*_{yy^-_1}(a') - 1) \Big|y\neq 1\Big] \nonumber \\
&\geq \lambda_1\,(\delta_{\bold{1}}\,(a-a')) - (1-\lambda_1)\,\gamma_C\,\sqrt{k}\,\Delta_2\,(a-a') \label{eq:Slowerbound} \\
& = (a-a')\,(-\gamma_C\,\sqrt{k}\,\Delta_2 + \lambda_1\,(\delta_{\bold{1}} + \gamma_C\,\sqrt{k}\,\Delta_2)) \nonumber \\
&> (a-a')\,(-\gamma_C\,\sqrt{k}\,\Delta_2 + \gamma_C\,\sqrt{k}\,\Delta_2) \label{eq:fromlambda1threshold} \\
&=0, \nonumber
\end{align}
where (\ref{eq:Slowerbound}) follows from (\ref{eq:phislopelowerbound}) and
Lemma~\ref{lemma:Astarbound} together with the fact that $s\geq - |s|$ for all $s \in \RR$, and (\ref{eq:fromlambda1threshold}) follows from (\ref{eq:SCLMClambda1suffcond}). 
Thus, for all $a\in [-1,1]$, $S(A^*(a))$ is a strictly increasing function of the variable $a$ and is minimized when $a = -1$. When $a = -1$, $b = (a^2(C-1)-1)/(C-2) = 1$. Then, $\forall i \in \Ccal\setminus\{1\}$, $(\mu^*_i)^\top \mu^*_1 = a = -1$. Since for all $j\in \Ccal$ we have $||\mu^*_j||^2 = 1$, it follows from the alignment conditions for equality in the Cauchy-Schwartz inequality that for all $i \in \Ccal\setminus\{1\}$, $\mu^*_i = -\mu^*_1$. 
\end{proof}

\begin{corollary}\label{cor:SCLInfoNCElambda1threshold}
For the InfoNCE loss function, condition (\ref{eq:SCLMClambda1suffcond}) for minority collapse in Theorem~\ref{theorem:SCLminoritycollapse} is satisfied if
\[
\lambda_1 \in [\tau_C,1), \text{ where } \tau_C :=  \frac{1}{1 + \tfrac{2}{\gamma_C(1+e^2)}}.
\label{eq:SCLInfoNCElambda1threshold} 
\]
Moreover, for all $C \geq 3$, $\tau_C \leq \tau_3 \approx 0.9438$. Thus, $\lambda_1 \geq 0.9438$ is a sufficient condition for minority collapse in the \SCL{} setting for the InfoNCE loss function, irrespective of the number of classes $C$ or the number of negative samples per anchor sample $k$.
\end{corollary}

\begin{proof}As in the proof of Corollary~\ref{cor:InfoNCElambda1threshold} in Appendix~\ref{apd:proof_InfoNCElambda1threshold},
\[
\Delta_2 = \frac{1}{2\sqrt{2}}.
\]
Moreover,
\[
\delta_{\bold{1}} = \bold{1}^\top\,\nabla\psi(-2\,\bold{1}) = \frac{ke^{-2}}{k+ ke^{-2}} = \frac{1}{1+e^{2}}.
\]
Plugging these into  (\ref{eq:SCLMClambda1suffcond}) we get
\[
\tau = \tau_C := \frac{1}{1+\frac{2}{\gamma_C} \frac{1}{1+e^2}}.
\]
The most conservative (maximum)
value of $\gamma_C$ occurs when $C$ is minimum, i.e., $C = 3$. When
$C = 3$, $\gamma_C = 4$, and $\tau_C \approx 0.9438$.
\end{proof}

\end{document}